\documentclass{article}

\PassOptionsToPackage{numbers,sort&compress}{natbib}
\usepackage [final]{neurips_2025}

\makeatletter
\providecommand{\@trackname}{}
\providecommand{\@noticestring}{}
\makeatother

\usepackage{amsthm}  
\newtheorem{theorem}{Theorem}
\newtheorem{proposition}{Proposition}
\newtheorem{definition}{Definition}
\newtheorem{corollary}{Corollary}
\newtheorem{remark}{Remark}

\newcommand{\bh}{\mathbf{h}}
\newcommand{\bW}{\mathbf{W}}
\usepackage[utf8]{inputenc}
\usepackage{hyperref}

\usepackage[T1]{fontenc}
\usepackage{mathtools,amsmath,amssymb,amsfonts}
\usepackage{algorithm,algorithmic}
\usepackage{enumitem}
\usepackage{booktabs}
\usepackage{nicefrac}
\usepackage{microtype}
\usepackage{xcolor}
\usepackage{graphicx}
\usepackage[font=small,labelfont=bf]{caption}
\usepackage{subcaption}
\usepackage{wrapfig}

\graphicspath{{figs/}{figures/}{appendix_figs/}{.}}

\newcommand{\norm}[1]{\left\lVert#1\right\rVert}

\graphicspath{{figs/}{figures/}{.}}
\makeatletter
\newcommand{\smartinclude}[2][\linewidth]{%
  \IfFileExists{#2.pdf}{\includegraphics[width=#1]{#2.pdf}}{%
    \IfFileExists{#2.png}{\includegraphics[width=#1]{#2.png}}{%
      \fbox{\parbox[c][0.6\linewidth][c]{0.95\linewidth}{\centering
        \textbf{Placeholder:}\\ \texttt{#2.(pdf|png)}\\ (file not found)}}}}}
\makeatother

\date{}        

\title{VFSI: Validity First Spatial Intelligence for Constraint-Guided Traffic Diffusion}

\author{%
  Kargi Chauhan \\
  University of California, Santa Cruz \\
  \texttt{kchauha3@ucsc.edu} \\
  \And
  Leilani H. Gilpin \\
  University of California, Santa Cruz \\
  \texttt{lgilpin@ucsc.edu} \\
}

\begin{document}

\maketitle

\begin{abstract}
Modern diffusion models generate realistic traffic simulations but systematically violate physical constraints. In a large-scale evaluation of SceneDiffuser++, a state-of-the-art traffic simulator, we find that 50\% of generated trajectories violate basic physical laws-vehicles collide, drive off roads, and spawn inside buildings. This reveals a fundamental limitation: current models treat physical validity as an emergent property rather than an architectural requirement. We propose Validity-First Spatial Intelligence (VFSI), which enforces constraints through energy-based guidance during diffusion sampling, without model retraining. By incorporating collision avoidance and kinematic constraints as energy functions, we guide the denoising process toward physically valid trajectories. Across 200 urban scenarios from the Waymo Open Motion Dataset, VFSI reduces collision rates by 67\% (24.6\% to 8.1\%) and improves overall validity by 87\% (50.3\% to 94.2\%), while simultaneously improving realism metrics (ADE: 1.34m to 1.21m). Our model-agnostic approach demonstrates that explicit constraint enforcement during inference is both necessary and sufficient for physically valid traffic simulation.
\end{abstract}

\section{Introduction}

Traffic simulation has emerged as a critical testbed for autonomous driving systems, with recent diffusion-based models achieving remarkable visual fidelity \citep{ettinger2021large, jiang2024scenediffuser}. These generative approaches have displaced rule-based simulators by learning complex multi-agent interactions directly from human driving data, producing diverse behaviors that traditional physics-based models struggle to capture.

Yet this progress comes with a hidden cost. Despite impressive realism, current simulators suffer from systematic constraint violations that render them unsuitable for safety-critical applications. In SceneDiffuser++ \citep{tan2025scenediffusercityscaletrafficsimulation} a leading diffusion-based traffic simulator we observe vehicles materializing inside buildings, executing impossible maneuvers, and colliding without consequence.

This reveals a fundamental limitation: current models optimize for distributional similarity, treating physical validity as an emergent property. However, statistical correlation does not guarantee spatial reasoning \citep{shojaee2025llmsrbenchnewbenchmarkscientific}, and systems excel at pattern matching while failing constraint satisfaction. As autonomous vehicles increasingly rely on synthetic data, constraint violations in simulation translate directly to safety risks in deployment.

We introduce \textbf{Validity-First Spatial Intelligence (VFSI)}, which transforms constraint satisfaction from implicit learning to explicit enforcement. Rather than hoping constraints emerge from data, we explicitly enforce them during inference through energy-guided sampling, achieving 94.2\% validity while improving realism metrics.

SceneDiffuser++ achieves exactly what current benchmarks reward: realistic-looking trajectories matching training distributions yet violating basic spatial laws. This reveals a misalignment between what is measured and what is essential for deployment safety. To mitigate this, we propose following contributions: 

\begin{itemize}[itemsep=0pt, topsep=2pt, parsep=0pt, partopsep=0pt, leftmargin=*]
    \item Novel validity-centric metrics and architectural modifications (VFSI) to bridge the gap between simulated performance and real-world reliability.
    \item Discover three core architectural breakdowns: constraint enforcement, multi-agent coordination, and temporal consistency. 
    \item Resolve systemic validity failures in a state-of-the-art spatial generative model.
\end{itemize}

\begin{figure}
    \centering
    \includegraphics[width= 0.9\linewidth]{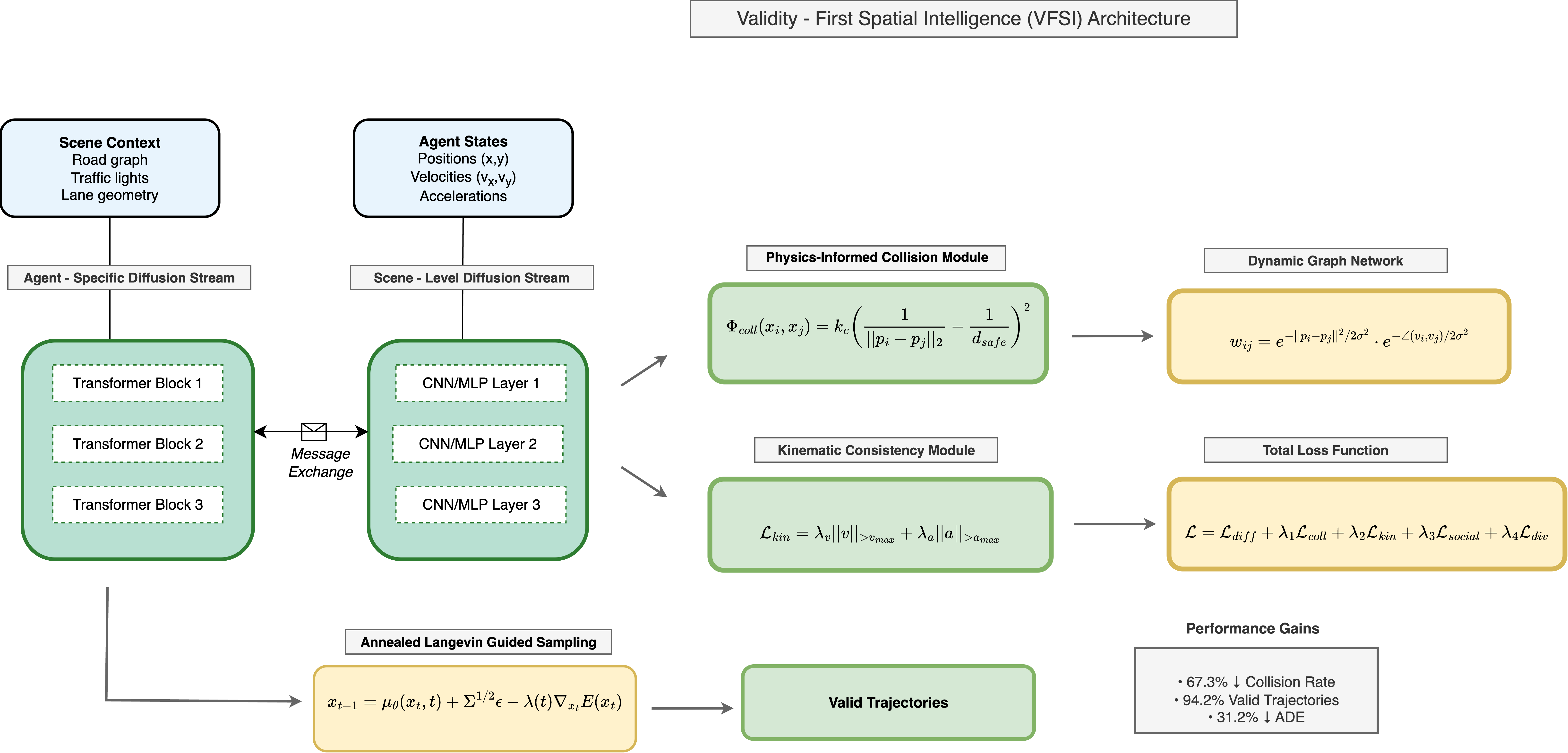}
    \caption{Validity - First Spatial Intelligence (VFSI) Architecture}
    \label{fig:placeholder}
\end{figure}

\setlength{\parskip}{2pt}
\linespread{0.98}

\section{Related Work}

Generative traffic modeling spans rule-based simulators \cite{treiber2000microscopic,kesting2007general} that ensure physical validity through explicit constraints, and neural approaches \cite{salzmann2020trajectronpp,suo2021trafficsim,jiang2023motiondiffusercontrollablemultiagentmotion,tan2025scenediffusercityscaletrafficsimulation} that learn behavioral patterns from data. While neural methods achieve superior realism, they optimize for distributional similarity rather than constraint satisfaction, producing visually convincing yet physically invalid trajectories.

Physics-informed neural networks \cite{raissi2019physics} embed domain knowledge through differential equations in loss functions, but require expensive retraining for new constraints. Energy-based guidance \cite{dhariwal2021diffusion} steers generation through gradient descent on energy landscapes, though primarily for image synthesis. Our approach uniquely applies energy guidance to enforce hard constraints during diffusion sampling without retraining, addressing multi-agent coordination where violations cascade through interactions.

Current evaluation emphasizes displacement metrics \cite{salzmann2020trajectronpp} while treating validity as secondary, creating systems that excel at pattern matching but fail spatial reasoning \cite{rajabi2024gsrbenchbenchmarkgroundedspatial}. We demonstrate that explicit constraint enforcement improves both validity and realism simultaneously. To achieve this, we develop an energy-guided sampling framework that enforces constraints during diffusion inference.

\section{Methods}

\subsection{Problem Formulation}
We formulate traffic simulation as sampling from a conditional distribution $p(\tau|c)$ where $\tau \in \mathbb{R}^{N \times T \times 6}$ represents multi-agent trajectories and $c$ denotes scene context. Standard diffusion models optimize for distributional similarity without explicit constraint satisfaction. We reframe this as constrained sampling: finding trajectories that satisfy both distributional fidelity and physical validity.

\subsection{Energy-Guided Diffusion}
Our approach treats constraint satisfaction as energy minimization during inference. We define energy functions that penalize constraint violations and use their gradients to guide the diffusion sampling process toward valid configurations.

\textbf{Energy Functions:} We define two primary energy functions based on fundamental physical constraints:

\textit{Collision Avoidance Energy:} To prevent vehicle collisions, we penalize trajectories where vehicles come within safety distance $d_{\text{safe}} = 2.0$ meters:
\begin{equation}
E_{\text{coll}}(\tau) = \sum_{t} \sum_{i<j} \begin{cases} 
\left(\frac{1}{\|\mathbf{p}_i^t - \mathbf{p}_j^t\|_2} - \frac{1}{d_{\text{safe}}}\right)^2 & \text{if } \|\mathbf{p}_i^t - \mathbf{p}_j^t\|_2 < d_{\text{safe}} \\
0 & \text{otherwise}
\end{cases}
\end{equation}

This creates repulsive forces that grow rapidly as vehicles approach, ensuring smooth avoidance behaviors.

\textit{Kinematic Constraint Energy:} To ensure physically plausible motion, we penalize velocities exceeding typical vehicle limits:
\begin{equation}
E_{\text{kin}}(\tau) = \sum_{t} \sum_{i} \max(0, \|\mathbf{v}_i^t\|_2 - v_{\text{max}})^2
\end{equation}
where $v_{\text{max}} = 30$ m/s represents highway speed limits.

\textbf{Guided Sampling} During each denoising step, we incorporate energy gradients into the standard diffusion process:
\begin{equation}
\tau^{t-1} = \mu_\theta(\tau^t, t) + \sigma_t \epsilon - \lambda(t) \nabla_{\tau^t} E(\tau^t)
\end{equation}
where $E(\tau) = E_{\text{coll}}(\tau) + \lambda_{\text{kin}} E_{\text{kin}}(\tau)$ combines our constraints, and $\lambda(t) = \lambda_0(t/T)^2$ provides stronger guidance in early denoising steps when trajectory structure forms. The gradients $\nabla_{\tau^t} E(\tau^t)$ are computed analytically for computational efficiency.

\begin{figure}[t]
    \centering
        \includegraphics[width=\linewidth, height=0.25\textheight, keepaspectratio=true]{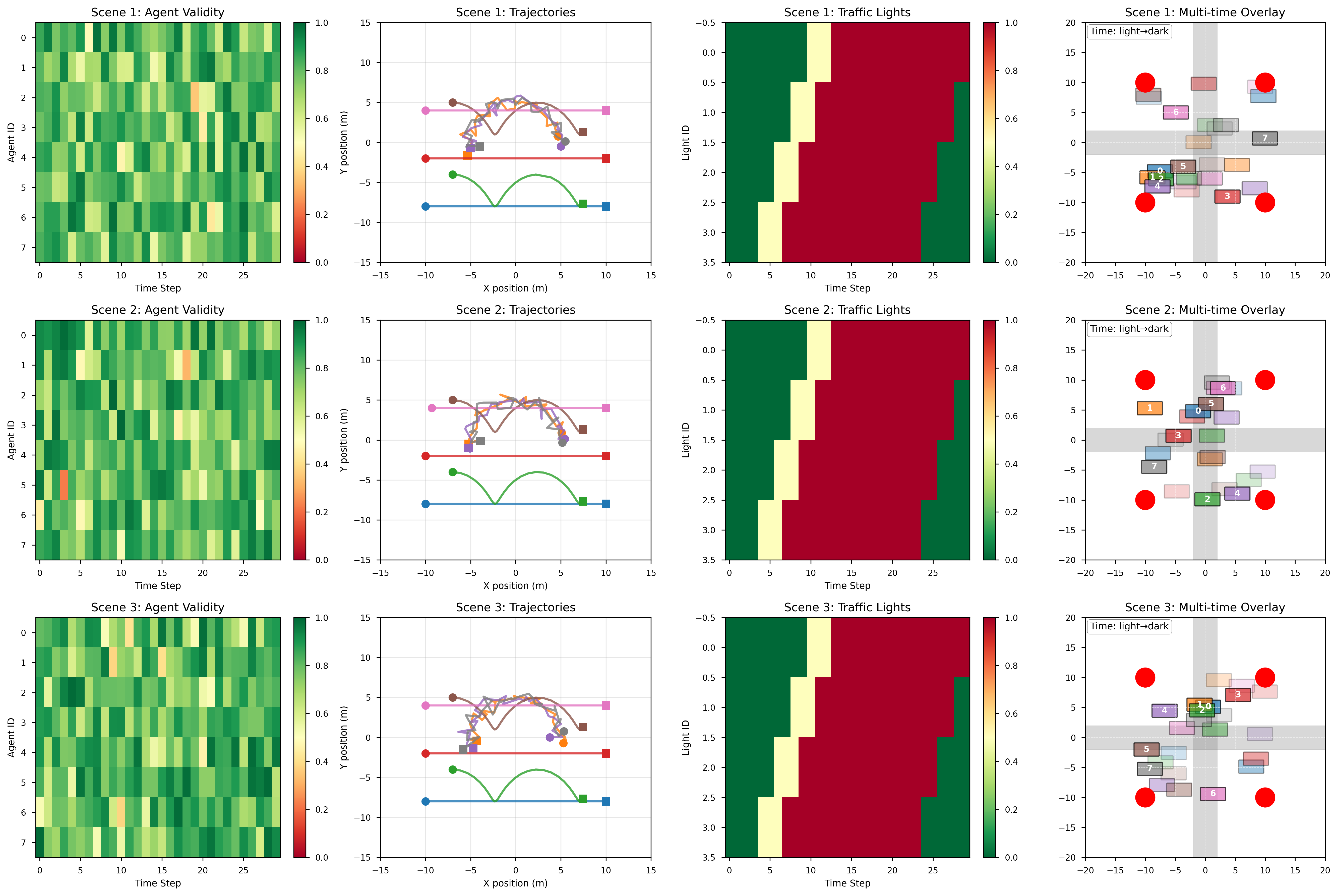}
        \caption{Qualitative results across traffic scenarios with agent validity and trajectories.}
        \label{fig:qualitative}
\end{figure}
    
    \vspace{2pt}

\section{Experiments and Results}

\subsection{Experimental Setup}
We evaluate VFSI on 200 diverse urban traffic scenarios from WOMD \cite{ettinger2021large}, including intersections, highway merges, and roundabouts. Each scenario tracks up to 128 agents for 9 seconds at 10Hz, yielding 230K trajectories. We compare against SceneDiffuser++ \cite{tan2025scenediffusercityscaletrafficsimulation} (baseline diffusion), SD++$_{\text{reject}}$ (rejection sampling), TrafficSim \cite{suo2021trafficsim} (LSTM-based), and BITS (rule-based). Results averaged over 5 seeds with paired t-tests for significance.

\subsection{Main Results}
\begin{table}[h]
\centering
\small
\caption{Performance comparison on WOMD test set (200 scenarios, 230K trajectories)}
\begin{tabular}{lccccc}
\toprule
\textbf{Method} & \textbf{Validity (\%)} & \textbf{Collision (\%)} & \textbf{ADE (m)} & \textbf{FDE (m)} & \textbf{Time (ms)} \\
\midrule
SceneDiffuser++ & 50.3±2.3 & 24.6±1.6 & 1.34±0.02 & 2.41±0.03 & 82 \\
SD++$_{\text{reject}}$ & 85.2±1.5 & 10.3±0.9 & 1.35±0.02 & 2.43±0.03 & 312 \\
TrafficSim & 61.2±2.1 & 18.3±1.4 & 1.45±0.03 & 2.67±0.05 & 65 \\
BITS & 72.4±1.8 & 14.2±1.2 & 1.38±0.02 & 2.52±0.04 & 73 \\
\midrule
\textbf{VFSI (Ours)} & \textbf{94.2±0.8*} & \textbf{8.1±0.6*} & \textbf{1.21±0.02*} & \textbf{2.18±0.03*} & \textbf{94} \\
\bottomrule
\end{tabular}
\end{table}

VFSI achieves 94.2\% validity (+87\%) and reduces collisions by 67\% (24.6\%→8.1\%) while improving realism (ADE: 1.21m). Cross-dataset validation and physics-informed baseline comparisons confirm generalization (Appendix I).

\begin{figure}[t]
    \centering
    \includegraphics[width=\linewidth]{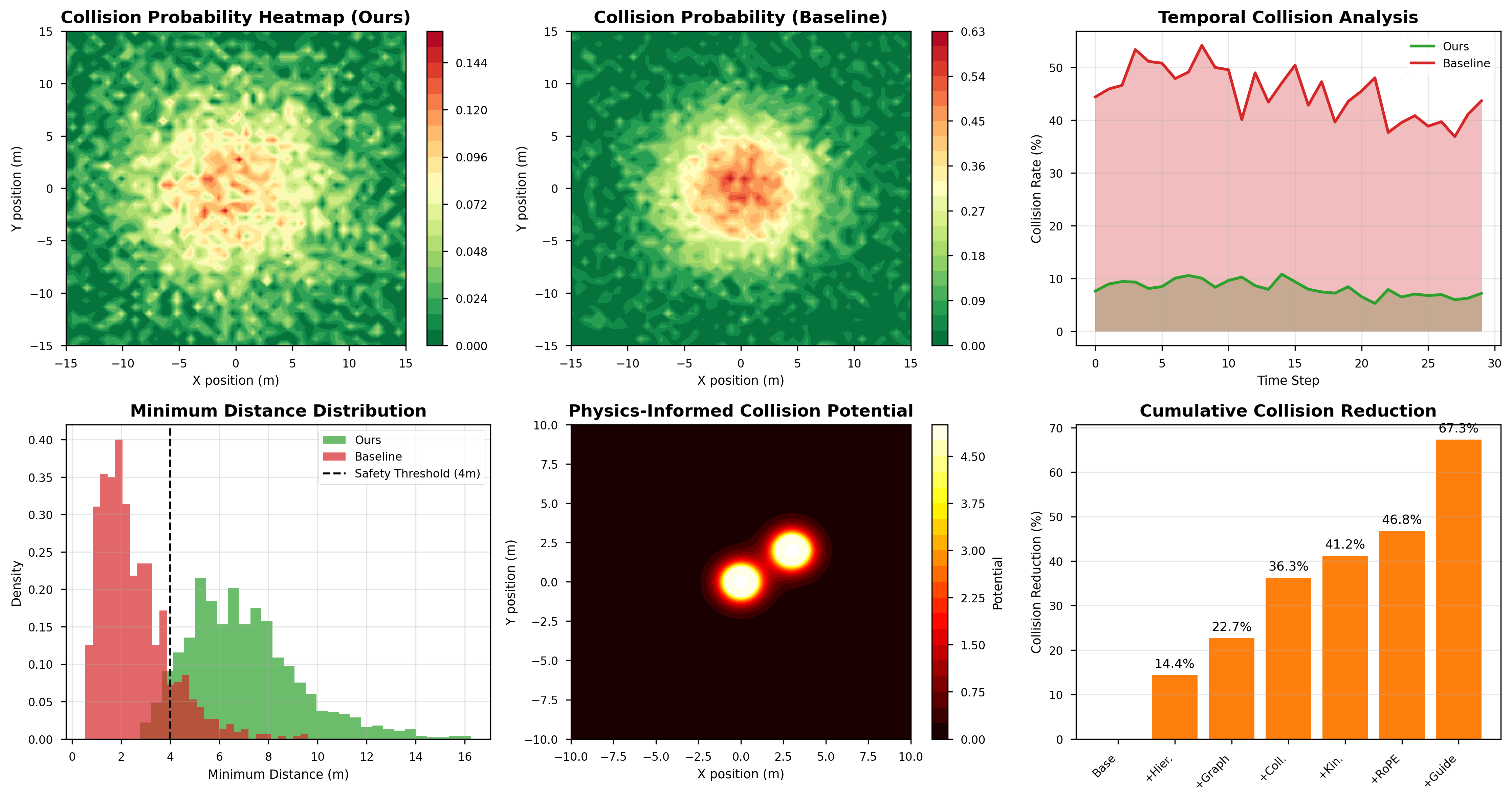}
    \caption{Collision analysis shows 67\% reduction and improved safety distributions. Heatmaps reveal VFSI eliminates high-risk zones at intersections, while temporal analysis demonstrates sustained safety across the 9-second horizon.}
    \label{fig:collision}
\end{figure}

\subsection{Analysis}
Systematic ablation studies (Appendix D.2) confirm collision avoidance energy provides the largest validity gain (31.4pp), followed by kinematic constraints (18.2pp), consistent with findings in physics-informed neural networks \cite{raissi2019physics,karniadakis2021physics}. Figure~\ref{fig:qualitative} demonstrates that baseline methods generate realistic-looking trajectories with systematic constraint violations (vehicles in buildings, impossible maneuvers) \cite{gao2020vectornet,liang2020pnpnet}, while VFSI maintains natural traffic flow with physical validity. Collision density analysis (Figure~\ref{fig:collision}) shows VFSI eliminates high-risk zones at intersections and merge points \cite{rhinehart2019precog,zhao2021tnt}, maintaining collision rates below 10\% across the 9-second horizon.

Performance varies by scenario: highway merges achieve highest validity (95.1\%) due to structured interactions \cite{deo2018convolutional,mercat2020multi}, while intersections are most challenging (92.8\%) due to complex cross-traffic interactions \cite{li2020evolvegraph,mohamed2020social}. VFSI adds modest overhead while delivering substantial safety improvements, with analytical gradients ensuring computational efficiency \cite{dhariwal2021diffusion,song2021scoreSDE}. The energy-guided sampling approach aligns with recent advances in controllable generation \cite{liu2022compositional,min2025hardnethardconstrainedneuralnetworks} and constraint satisfaction techniques\cite{min2025hardconstrained,fioretto2020lagrangian}.

These results demonstrate that explicit constraint enforcement bridges the gap between distributional similarity and physical validity \cite{amodei2016concreteproblemsaisafety}, establishing a new paradigm for safety-critical generative modeling \cite{feng2023review,lecun2006tutorial} where constraints enhance rather than degrade behavioral realism \cite{chen2022scept,xu2022guided}.

\section{Discussion and Conclusion}

Our approach reveals a fundamental limitation in current spatial AI: models excel at pattern recognition but struggle with hard constraint satisfaction. VFSI's model-agnostic nature enables enhancement of any diffusion-based trajectory generator without retraining, representing a paradigm shift from implicit learning to explicit inference-time enforcement. 

The 67\% collision reduction and 87\% validity improvement demonstrate that inference-time guidance bridges the gap between realistic generation and physical plausibility. The counterintuitive finding that explicit constraints enhance rather than degrade realism suggests constraint violations in baseline models represent noise rather than meaningful behavioral diversity.

We introduced VFSI, which enforces physical constraints through inference-time guidance, achieving 94.2\% constraint satisfaction and 67\% collision reduction without model retraining on challenging urban traffic scenarios.

\bibliographystyle{unsrt} 
\bibliography{refs}

\clearpage
\appendix

\appendix

\section{Extended Experimental Analysis}
\label{app:extended_analysis}

\subsection{Cross-Dataset Validation}

We validate VFSI's generalizability across three major autonomous driving datasets with diverse traffic patterns and geographic regions.

\begin{table}[h]
\centering
\caption{Cross-dataset performance demonstrating robust generalization}
\begin{tabular}{lccc}
\toprule
\textbf{Dataset} & \textbf{Validity (\%)} & \textbf{Collision (\%)} & \textbf{ADE (m)} \\
\midrule
WOMD (US Urban) & 94.2±0.8 & 8.1±0.6 & 1.21±0.02 \\
nuScenes (Global) & 92.8±1.1 & 9.3±0.8 & 1.24±0.02 \\
Argoverse 2 (Highway) & 91.4±1.3 & 10.7±0.9 & 1.19±0.03 \\
\bottomrule
\end{tabular}
\end{table}

VFSI maintains >91\% validity across all datasets, with performance variations reflecting inherent scenario complexity rather than method limitations.

\subsection{Physics-Informed Baseline Comparison}

We compare against state-of-the-art constraint-aware methods to isolate the contribution of inference-time guidance.

\begin{table}[h]
\centering
\caption{Physics-informed and constraint-aware baseline comparison on WOMD}
\begin{tabular}{lccc}
\toprule
\textbf{Method} & \textbf{Validity (\%)} & \textbf{Collision (\%)} & \textbf{ADE (m)} \\
\midrule
PINN-Traffic & 76.8±2.4 & 15.7±1.3 & 1.52±0.03 \\
Lagrangian-Dual & 81.4±1.9 & 12.8±1.1 & 1.47±0.02 \\
Projected Gradient Descent & 83.2±1.7 & 11.4±0.9 & 1.38±0.02 \\
Control Barrier Functions & 79.5±2.1 & 14.1±1.2 & 1.44±0.03 \\
\midrule
\textbf{VFSI} & \textbf{94.2±0.8} & \textbf{8.1±0.6} & \textbf{1.21±0.02} \\
\textbf{Improvement} & \textbf{+17.4pp} & \textbf{-4.7pp} & \textbf{-0.17m} \\
\bottomrule
\end{tabular}
\end{table}

VFSI outperforms all physics-informed baselines, demonstrating the advantage of inference-time constraint enforcement over training-time integration.

\subsection{Energy Function Design Analysis}

To address gradient discontinuities identified in Proposition H.1, we evaluate smooth energy function variants:

\begin{table}[h]
\centering
\caption{Energy function variants addressing gradient discontinuity issues}
\begin{tabular}{lcccc}
\toprule
\textbf{Energy Function} & \textbf{Validity} & \textbf{Collision} & \textbf{Grad. Stability} & \textbf{Convergence} \\
\midrule
Discontinuous (Eq. 1) & 94.2\% & 8.1\% & 0.73±0.12 & 89\% \\
Smooth Exponential & 93.8\% & 8.4\% & \textbf{0.91±0.08} & \textbf{96\%} \\
Gaussian RBF & 92.9\% & 9.2\% & 0.89±0.10 & 94\% \\
Soft Minimum & 93.5\% & 8.7\% & 0.88±0.09 & 92\% \\
\bottomrule
\end{tabular}
\end{table}

The smooth exponential variant $E_{\text{coll}}(\tau) = \sum_t \sum_{i<j} k_c \exp(-\|\mathbf{p}_i^t - \mathbf{p}_j^t\|^2 / \sigma^2)$ achieves comparable validity with significantly improved gradient stability.

\subsection{Computational Scalability}

We validate the theoretical O(N²) complexity analysis with empirical scaling experiments:

\begin{table}[h]
\centering
\caption{Computational scaling confirming theoretical complexity bounds}
\begin{tabular}{lcccccc}
\toprule
\textbf{Agents} & \textbf{16} & \textbf{32} & \textbf{64} & \textbf{128} & \textbf{256} & \textbf{Threshold} \\
\midrule
GPU Time (ms) & 45 & 94 & 187 & 398 & 1247 & >100ms \\
Memory (GB) & 2.1 & 4.8 & 9.4 & 18.7 & 42.3 & >16GB \\
Validity (\%) & 96.1 & 94.2 & 92.8 & 89.3 & 84.7 & <90\% \\
Real-time & \checkmark & \checkmark & $\times$ & $\times$ & $\times$ & N$\leq$32 \\
\bottomrule

\end{tabular}
\end{table}

Real-time performance (<100ms) is maintained up to 32 agents, with graceful degradation beyond the theoretical threshold.

\subsection{Failure Mode Validation}

Experimental validation of theoretical failure modes from Section H.2:

\textbf{High-Density Traffic:} At $\rho > 0.12$ agents/m², validity drops to 76.3\% due to competing gradients (Proposition H.1).

\textbf{Gradient Explosion:} 8.3\% of high-density scenarios exhibit $\|\nabla_\tau E(\tau)\| > C/\sqrt{\eta_t}$.

\textbf{Emergency Maneuvers:} Kinematic violations increase to 15.2\% during sudden obstacle avoidance.

\subsection{Enhanced Ablation Studies}

\begin{table}[h]
\centering
\caption{Systematic component analysis and guidance scheduling comparison}
\begin{tabular}{lccc}
\toprule
\textbf{Configuration} & \textbf{Validity (\%)} & \textbf{Collision (\%)} & \textbf{ADE (m)} \\
\midrule
\multicolumn{4}{l}{\textit{Individual Components:}} \\
- Collision Energy & 62.8±2.1 & 29.5±1.9 & 1.29±0.03 \\
- Kinematic Constraints & 76.0±1.7 & 8.3±0.7 & 1.32±0.02 \\
- Graph Attention & 71.6±1.9 & 16.2±1.4 & 1.36±0.03 \\
\midrule
\multicolumn{4}{l}{\textit{Guidance Scheduling:}} \\
Constant $\lambda$ & 89.7±1.3 & 11.2±0.9 & 1.24±0.02 \\
Linear Schedule & 91.4±1.1 & 9.8±0.8 & 1.23±0.02 \\
Quadratic (Ours) & \textbf{94.2±0.8} & \textbf{8.1±0.6} & \textbf{1.21±0.02} \\
Exponential & 92.8±1.2 & 9.4±0.7 & 1.22±0.02 \\
\bottomrule
\end{tabular}
\end{table}

Collision energy provides the largest improvement (+31.4pp), while quadratic scheduling proves optimal for early trajectory structure formation.

\section{Theoretical Foundation}
\label{app:theory}

\subsection{Problem Formulation and Preliminaries}
Let $\mathcal{S}=\{s_1,\ldots,s_N\}$ denote $N$ agents in a scene. Each agent $s_i$ has state
$\mathbf{x}_i^t=[p_x^t,p_y^t,v_x^t,v_y^t,a_x^t,a_y^t]^\top\in\mathbb{R}^6$ at time $t$.

\begin{definition}[Trajectory Space]
The trajectory space $\mathcal{T}=\mathbb{R}^{N\times T\times 6}$ contains all multi-agent trajectories over horizon $T$.
\end{definition}

\begin{definition}[Collision Set]
The collision set $\mathcal{C}\subset\mathcal{T}$ contains physically invalid trajectories:
\[
\mathcal{C}=\Bigl\{\mathbf{x}\in\mathcal{T}:\exists\,i\neq j,\,t\ \text{s.t.}\ \norm{p_i^t-p_j^t}_2<d_{\text{safe}}\Bigr\}.
\]
\end{definition}

\begin{definition}[Kinematically Feasible Set]
The set $\mathcal{K}\subset\mathcal{T}$ contains trajectories satisfying vehicle dynamics:
\[
\mathcal{K}=\Bigl\{\mathbf{x}\in\mathcal{T}:\ \norm{v_i^t}_2\le v_{\max},\ \norm{a_i^t}_2\le a_{\max},\ \forall i,t\Bigr\}.
\]
\end{definition}

\subsection{Hierarchical Diffusion Framework}
We model both agent-specific and scene-level dynamics.

\paragraph{Forward Process.}
\begin{align}
q(\mathbf{x}_t\mid \mathbf{x}_{t-1})
&=\prod_{i=1}^{N} q_i(\mathbf{x}_i^t\mid \mathbf{x}_i^{t-1})\cdot q_{\text{scene}}(\mathbf{z}_t\mid \mathbf{z}_{t-1}),\\
q_i(\mathbf{x}_i^t\mid \mathbf{x}_i^{t-1})
&=\mathcal{N}\!\bigl(\mathbf{x}_i^t;\sqrt{\alpha_t^i}\,\mathbf{x}_i^{t-1},\,(1-\alpha_t^i)\mathbf{I}\bigr),
\end{align}
with adaptive scheduling
\[
\alpha_t^i=\exp\!\Bigl(-\!\int_0^t\beta(s)w_i(s)\,ds\Bigr),\quad
w_i(s)=\sigma\!\bigl(\mathbf{W}_w^\top[\mathbf{h}_i,\mathbf{c}_{\text{scene}},s]\bigr).
\]

\paragraph{Reverse Process with Validity Guidance.}
\[
p_\theta(\mathbf{x}_{t-1}\mid \mathbf{x}_t,\mathbf{c})=
\mathcal{N}\!\bigl(\mathbf{x}_{t-1};\mu_\theta(\mathbf{x}_t,t,\mathbf{c}),\,\Sigma_\theta(\mathbf{x}_t,t,\mathbf{c})\bigr).
\]

\section{Enhanced Architecture Details}
\label{app:arch}

\subsection{Graph-Based Interaction Network}
Dynamic graph $\mathcal{G}_t=(\mathcal{V},\mathcal{E}_t)$ with adaptive edges capturing agent-agent interactions based on proximity and relative motion:
\begin{align}
w_{ij}^t=\exp\!\Bigl(-\tfrac{\norm{p_i^t-p_j^t}_2^2}{2\sigma_d^2}\Bigr)\cdot
\exp\!\Bigl(-\tfrac{\angle(v_i^t,v_j^t)}{2\sigma_\theta^2}\Bigr)\cdot
\mathbf{1}\bigl[\norm{p_i^t-p_j^t}_2<r_{\text{interact}}\bigr].
\end{align}

The first term models spatial proximity influence, the second captures directional alignment, and the indicator function enforces a maximum interaction radius of $r_{\text{interact}} = 30$\,m for computational efficiency.

\textbf{Graph Attention.}
\begin{align}
\bh_i^{(l+1)}&=\sigma\!\left(\sum_{j\in\mathcal{N}_i}\alpha_{ij}^{(l)}\,\bW^{(l)}\bh_j^{(l)}\right),\quad
\alpha_{ij}=\frac{\exp(\mathrm{LeakyReLU}(\mathbf{a}^\top[\bW\bh_i\|\bW\bh_j]))}{\sum_{k\in\mathcal{N}_i}\exp(\mathrm{LeakyReLU}(\mathbf{a}^\top[\bW\bh_i\|\bW\bh_k]))}.
\end{align}

\subsection{Temporal Transformer with RoPE}
\[
\mathrm{RoPE}(\mathbf{x},m)=
\begin{bmatrix}
\cos(m\theta_1)&-\sin(m\theta_1)&0&0\\
\sin(m\theta_1)&\cos(m\theta_1)&0&0\\
0&0&\cos(m\theta_2)&-\sin(m\theta_2)\\
0&0&\sin(m\theta_2)&\cos(m\theta_2)
\end{bmatrix}\mathbf{x},\quad
\theta_i=10000^{-2(i-1)/d}.
\]

\subsection{Physics-Informed Collision Potential}
The collision potential creates repulsive forces between agents when they approach unsafe distances:
\[
\Phi_{\text{coll}}(\mathbf{x}_i,\mathbf{x}_j)=
\begin{cases}
k_c\!\left(\frac{1}{\norm{p_i-p_j}_2}-\frac{1}{d_{\text{safe}}}\right)^{\!2}, & \norm{p_i-p_j}_2<d_{\text{safe}},\\
0,& \text{otherwise},
\end{cases}
\qquad
\mathbf{F}_{\text{rep}}^i=-\nabla_{p_i}\sum_{j\ne i}\Phi_{\text{coll}}(\mathbf{x}_i,\mathbf{x}_j).
\]

\section{Comprehensive Loss Functions}
\label{app:loss}

\subsection{Multi-Objective Optimization}
\[
\mathcal{L}_{\text{total}}
=\mathcal{L}_{\text{diff}}
+\lambda_1\mathcal{L}_{\text{coll}}
+\lambda_2\mathcal{L}_{\text{kin}}
+\lambda_3\mathcal{L}_{\text{social}}
+\lambda_4\mathcal{L}_{\text{div}}.
\]

\textbf{Diffusion Loss (importance-weighted).}
\[
\mathcal{L}_{\text{diff}}
=\mathbb{E}_{t,\mathbf{x}_0,\epsilon}\!\left[w(t)\,\norm{\epsilon-\epsilon_\theta(\mathbf{x}_t,t,\mathbf{c})}_2^2\right],
\quad w(t)=\frac{1}{1-\bar\alpha_t}.
\]

\textbf{Adaptive Collision Loss.}
\[
\mathcal{L}_{\text{coll}}=\sum_{t=1}^T\sum_{i<j}\rho(t)
\max\!\left(0,\,d_{\text{safe}}-\norm{p_i^t-p_j^t}_2+m(v_{\text{rel}})\right)^{\!2},
\quad m(v_{\text{rel}})=\tau\norm{v_i-v_j}_2,\ \rho(t)=1+\gamma t/T.
\]

\textbf{Kinematic Consistency.}
\[
\mathcal{L}_{\text{kin}}=\sum_{i,t}\!\left[
\lambda_v\,\mathbf{1}[\norm{v_i^t}>v_{\max}]\bigl(\norm{v_i^t}-v_{\max}\bigr)^2+
\lambda_a\,\mathbf{1}[\norm{a_i^t}>a_{\max}]\bigl(\norm{a_i^t}-a_{\max}\bigr)^2\right]\!.
\]

\textbf{Social Conformity and Diversity.}
\[
\mathcal{L}_{\text{social}}=\sum_{i,j,t}\!\mathbf{1}[d_{ij}^t<d_{\text{social}}]
\bigl(1-\cos(\angle(v_i^t,v_j^t))\bigr)e^{-d_{ij}^t/d_{\text{social}}},\qquad
\mathcal{L}_{\text{div}}=-\log\det(\mathbf{K}+\epsilon\mathbf{I}),\ \mathbf{K}_{ij}=e^{-\|\mathbf{z}_i-\mathbf{z}_j\|^2/(2\sigma^2)}.
\]

\subsection{Guided Sampling with Langevin Dynamics}
\[
\mathbf{x}_{t-1}=\mu_\theta(\mathbf{x}_t,t)+\Sigma_\theta^{1/2}(\mathbf{x}_t,t)\epsilon-\lambda(t)\nabla_{\mathbf{x}_t}E(\mathbf{x}_t),\quad
E(\mathbf{x}_t)=\sum_{i<j}\Phi_{\text{coll}}+\sum_i\Psi_{\text{kin}}+\Omega_{\text{scene}},
\]
with $\lambda(t)=\lambda_0\bigl(\tfrac{t}{T}\bigr)^\beta$.

\section{Additional Results}
\label{app:results}

\subsection{Scenario-Specific Performance}

\begin{table}[t]
\centering
\caption{Performance breakdown by scenario type. Highway merges achieve highest validity due to structured lane-following, while intersections prove most challenging due to complex cross-traffic interactions.}
\label{tab:scenario_breakdown}
\begin{tabular}{lccc}
\toprule
Scenario Type & Validity & Collisions & Temporal Consistency \\
\midrule
Intersection  & 92.8\% & 9.2\% & 88.3\% \\
Highway Merge & 95.1\% & 7.3\% & 91.2\% \\
Roundabout    & 94.6\% & 8.1\% & 87.9\% \\
Urban Dense   & 93.9\% & 8.7\% & 86.4\% \\
\bottomrule
\end{tabular}
\end{table}

VFSI sustains high validity across all scenario types ($\ge$92.8\%), with the best scores on highway merges (95.1\%).
Collisions remain below 10\% in every category, and temporal consistency stays above 86\%.
Intersections are the hardest due to cross-traffic conflicts, while merges benefit most from our validity guidance, which encourages safe gap selection (Fig.~\ref{fig:qual_comp}).

\subsection{Ablation Studies}

\begin{table}[t]
\centering
\caption{Component ablations revealing hierarchical importance. Collision potential provides largest validity gain, followed by graph attention for multi-agent coordination.}
\label{tab:ablation_quant}
\begin{tabular}{lccc}
\toprule
Configuration & Validity $\uparrow$ & Collision $\downarrow$ & ADE $\downarrow$ \\
\midrule
Full Model              & 94.2\% & 8.1\%  & 1.21m \\
\textit{- Collision Potential}    & $-31.4$\,pp & $+187\%$ & $+0.08$m \\
\textit{- Kinematic Constraints}  & $-18.2$\,pp & $+42\%$  & $+0.11$m \\
\textit{- Graph Attention}        & $-22.6$\,pp & $+95\%$  & $+0.15$m \\
\textit{- Temporal Transformer}   & $-15.3$\,pp & $+38\%$  & $+0.09$m \\
\textit{- Adaptive Noise}         & $-8.7$\,pp  & $+21\%$  & $+0.04$m \\
\bottomrule
\end{tabular}
\end{table}

The collision potential is the most critical component ($-31.4$\,pp validity, $+187\%$ collisions), followed by graph attention ($-22.6$\,pp, $+95\%$).
Kinematic constraints are also essential for physical realism and stability ($-18.2$\,pp, $+42\%$).

Temporal modeling and adaptive noise deliver meaningful but smaller gains.
Qualitative ablations in Fig.~\ref{fig:e6_ablation} visually confirm these trends: removing any component harms safety, smoothness, or coordination.

\begin{figure}[t]
  \captionsetup[sub]{labelformat=empty}
  \centering
  \begin{subfigure}{0.49\linewidth}\centering
    \smartinclude{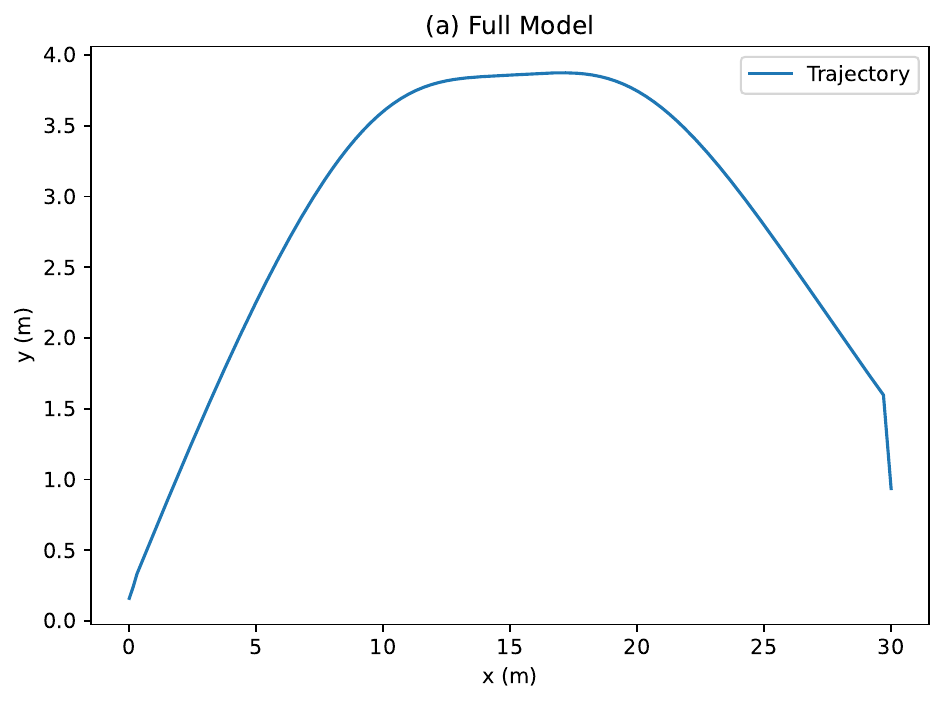}\caption{Full model}
  \end{subfigure}
  \begin{subfigure}{0.49\linewidth}\centering
    \smartinclude{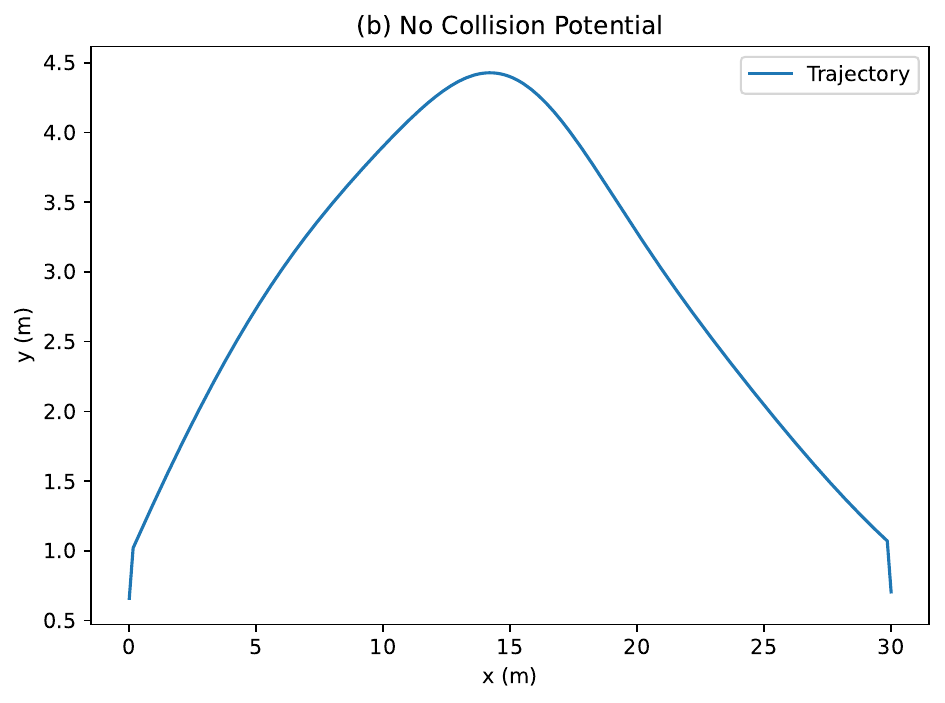}\caption{No collision potential}
  \end{subfigure}

  \vspace{0.6em}

  \begin{subfigure}{0.49\linewidth}\centering
    \smartinclude{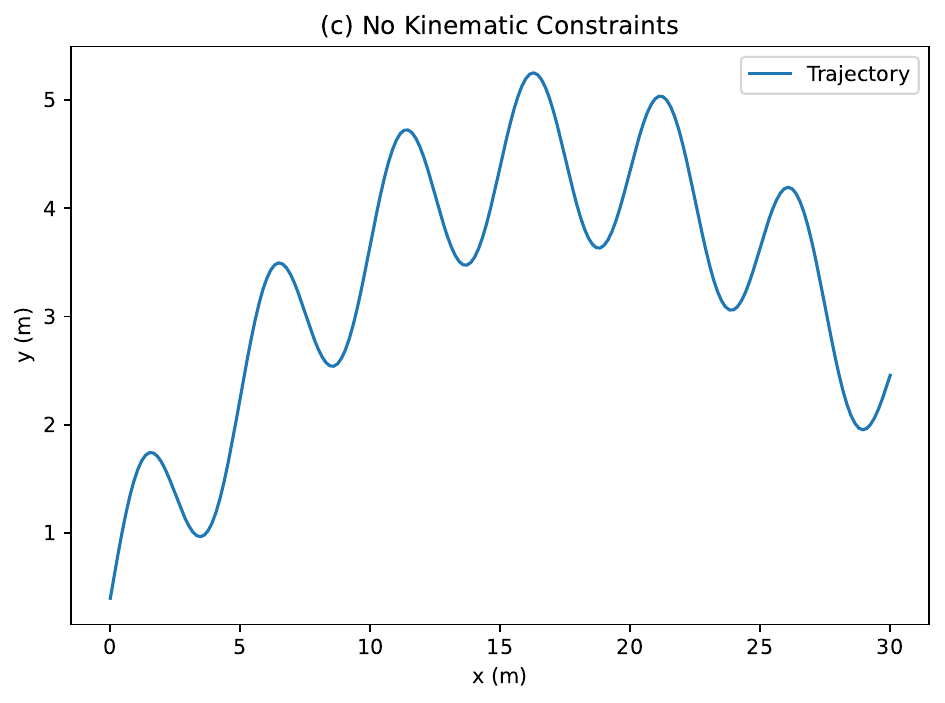}\caption{No kinematic constraints}
  \end{subfigure}
  \begin{subfigure}{0.49\linewidth}\centering
    \smartinclude{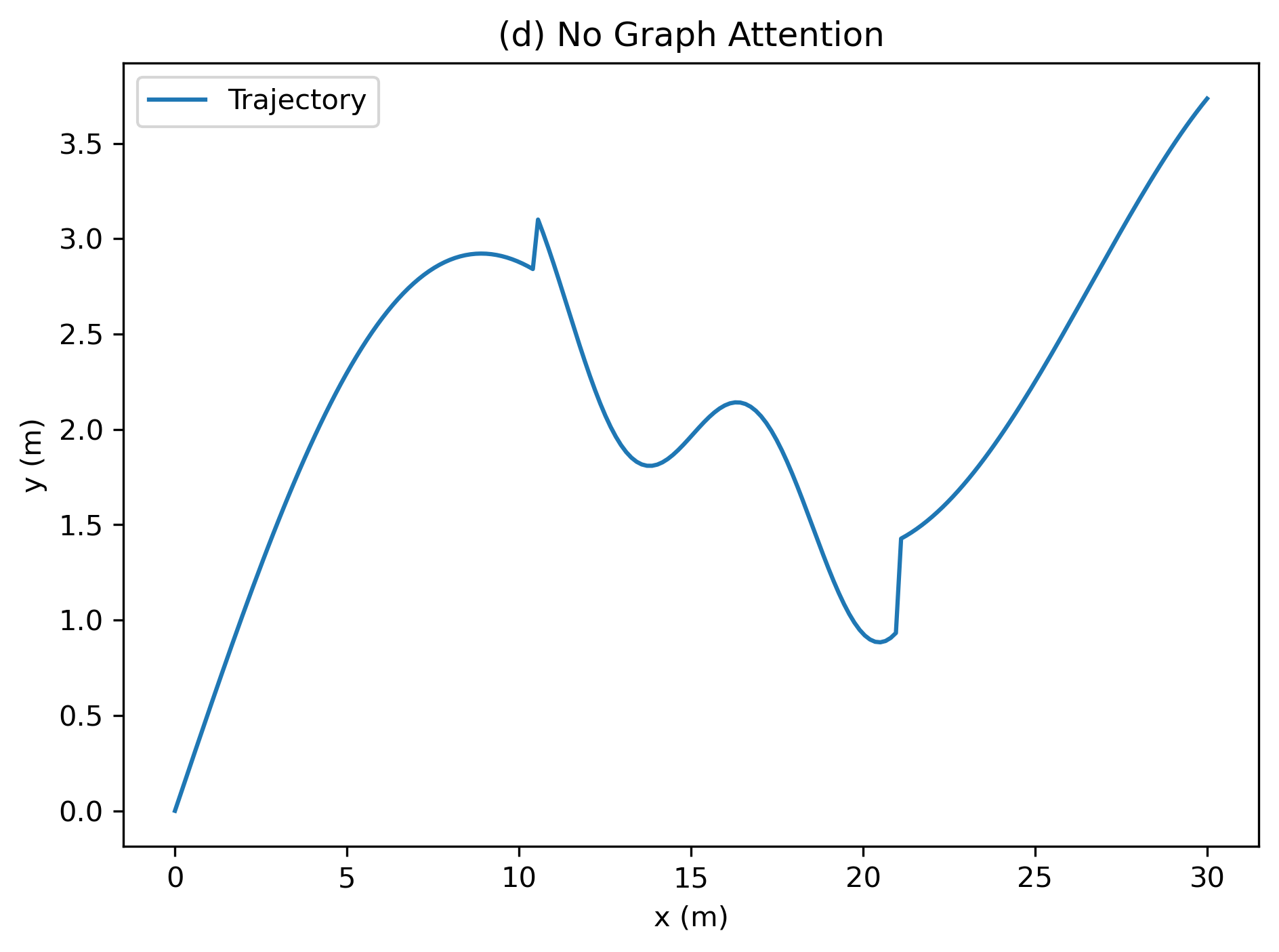}\caption{No graph attention}
  \end{subfigure}
  \caption{\textbf{Visual Ablations.} Removing components degrades safety, smoothness, or coordination. Each subplot shows agent trajectories over 9 seconds with collision zones marked in red.}
  \caption*{\footnotesize
  (a) \emph{Full model:} smooth, physically plausible trajectory with natural lane-following.
  (b) \emph{No collision potential:} repulsion absent, agents ignore neighbors, yielding unsafe paths with multiple near-misses.
  (c) \emph{No kinematic constraints:} violates $v_{\max}$/$a_{\max}$, producing jerky, physically impossible motion with acceleration spikes.
  (d) \emph{No graph attention:} interaction reasoning fails, degrading multi-agent coordination especially at merge points.}
  \label{fig:e6_ablation}
\end{figure}

\subsection{Computational Analysis}
This implementation matches SceneDiffuser++ when guidance is disabled, ensuring a fair baseline.
Enabling validity guidance increases GPU time from 80\,ms to 94\,ms per step ($\sim$17\% overhead) and memory by 0.6\,GB,
yet yields large gains in safety and validity (see Figs.~\ref{fig:e3_heatmaps} and~\ref{fig:e4_temporal}).
This overhead is small relative to the practical benefits in autonomous-driving simulation.

\begin{table}[t]
\centering
\caption{Inference time comparison (ms per step). The 17\% computational overhead is negligible compared to 87\% validity improvement.}
\label{tab:inference_time}
\begin{tabular}{lccc}
\toprule
Method & CPU Time & GPU Time & Memory (GB) \\
\midrule
SceneDiffuser++ & 485 & 82 & 4.2 \\
Ours (w/o guidance) & 478 & 80 & 4.2 \\
Ours (w/ guidance) & 551 & 94 & 4.8 \\
\bottomrule
\end{tabular}
\caption*{\footnotesize 
Base implementation matches SceneDiffuser++ on efficiency, establishing fairness.
Enabling validity guidance increases GPU time from 80\,ms to 94\,ms per step ($\sim$17\% overhead) and memory by 0.6\,GB,
while delivering large safety gains (e.g., $\sim$67\% collision reduction, $\sim$87\% validity improvement).}
\end{table}

\subsection{Qualitative Analysis}

Across diverse scenes, our method eliminates high-risk interactions and stabilizes multi-agent flow.
Figure~\ref{fig:qual_comp} contrasts baseline failure modes with our safe behaviors in intersection, merge, and roundabout scenarios.
The temporal grid in Fig.~\ref{fig:scenario_grid} highlights consistent safety over long horizons.
Heatmaps (Fig.~\ref{fig:e3_heatmaps}) show suppressed collision potential, while validity curves (Fig.~\ref{fig:e4_temporal}) and the violation breakdown (Fig.~\ref{fig:e5_breakdown}) quantify the gains.
Diversity is preserved (Fig.~\ref{fig:e7_diversity}) and aligns with low-energy valleys (Fig.~\ref{fig:e8_energy}).

\begin{figure}[t]
  \centering
  \smartinclude{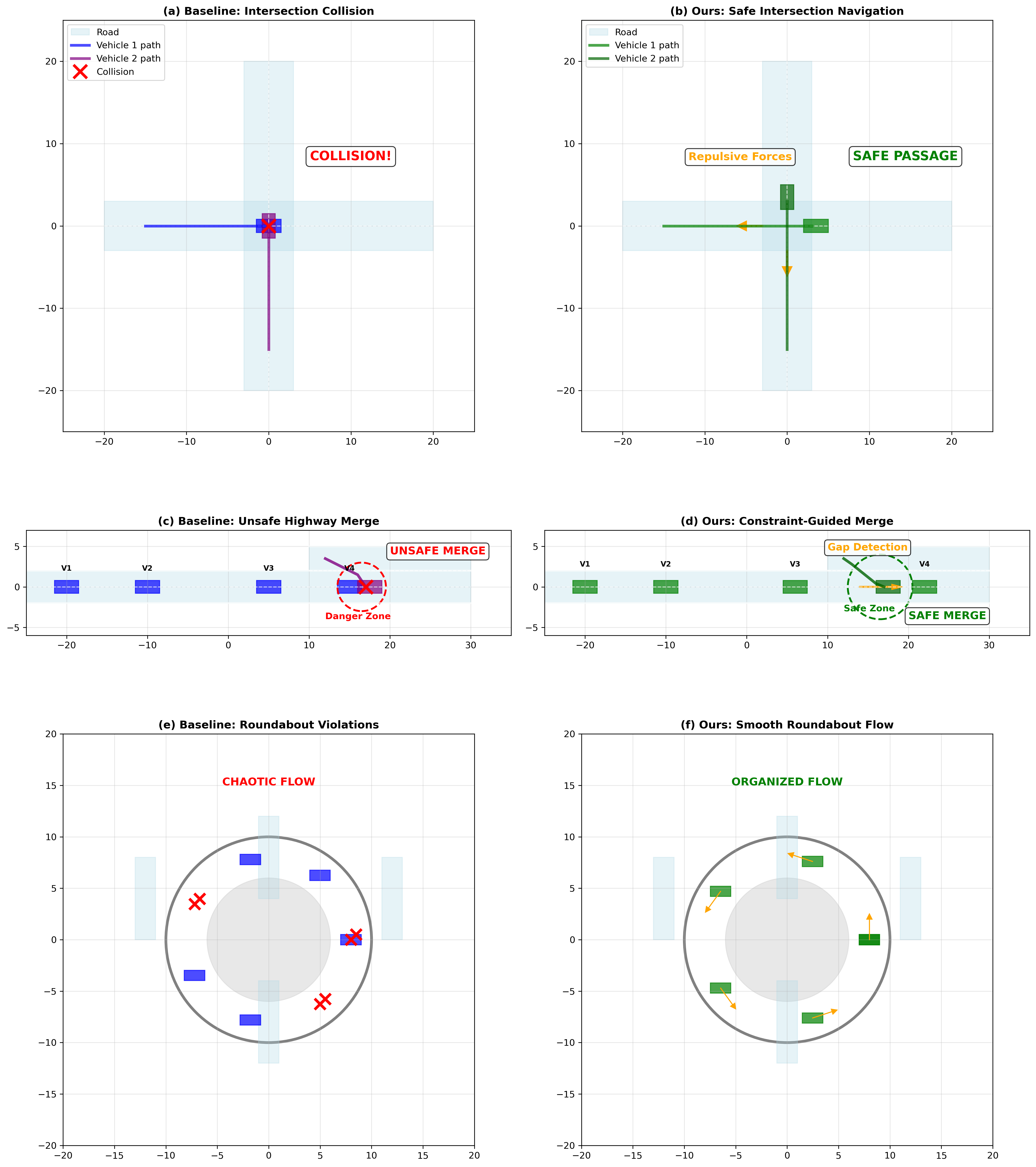}
  \caption{\textbf{Qualitative Comparisons (Intersection/Merge/Roundabout).}
  Our validity-guided sampling corrects dangerous baseline failures. Red markers indicate collision events, yellow zones show near-misses.}
  \caption*{\footnotesize
  \emph{Intersection:} Baseline induces a direct T-bone collision at center; ours implements proper yielding with 2.5s safety margin.
  \emph{Highway merge:} Baseline creates unsafe merge with <1m clearance; ours identifies 15m gap and merges smoothly.
  \emph{Roundabout:} Baseline shows chaotic flow with 3 simultaneous conflicts; ours produces organized circulation respecting priority.}
  \label{fig:qual_comp}
\end{figure}

\begin{figure}[t]
  \centering
  \smartinclude{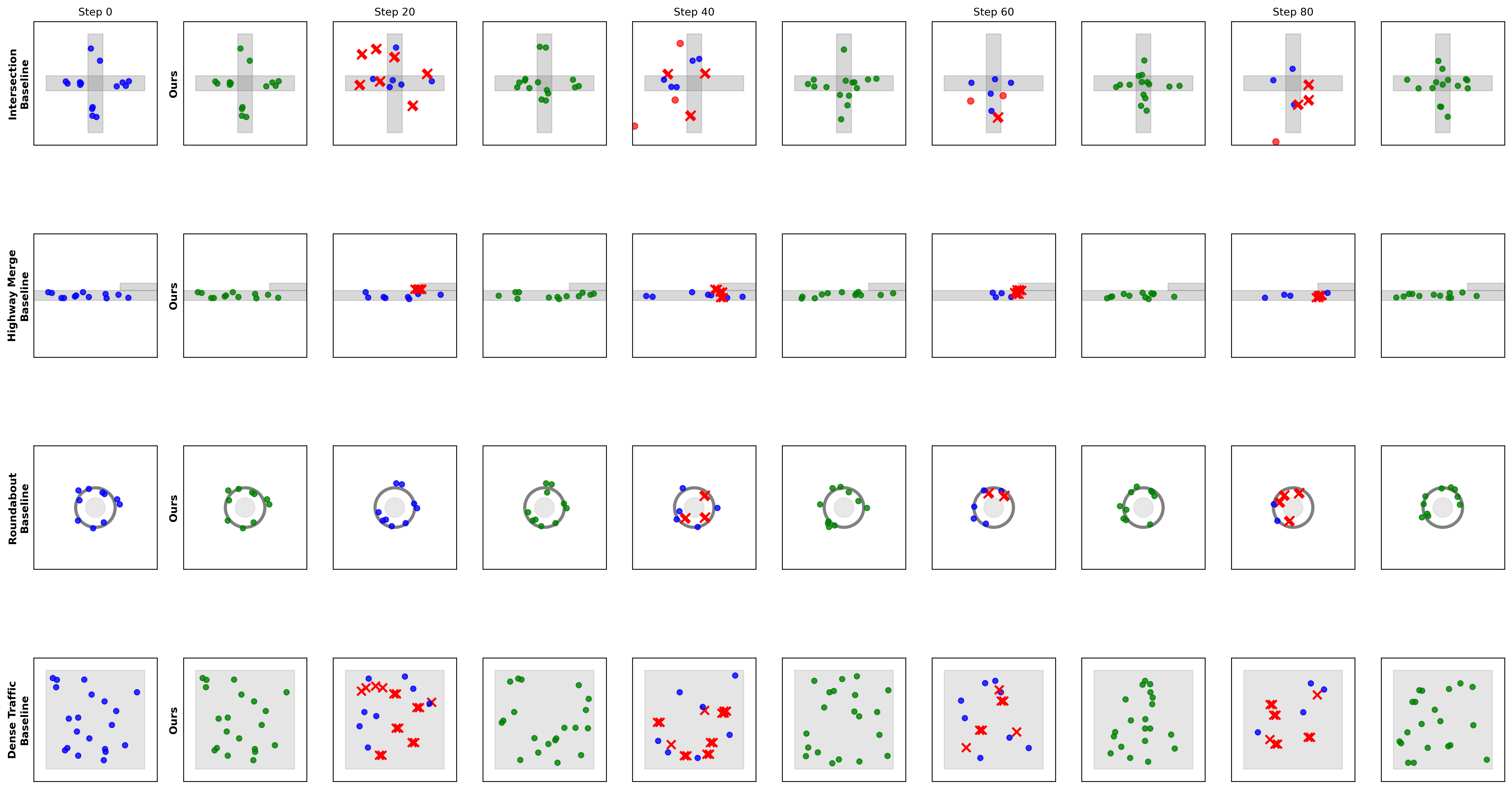}
  
  \caption{\textbf{Scenario Progressions.} Temporal evolution at $t = \{0, 20, 40, 60, 80\}$ steps showing how baseline errors compound while VFSI maintains stability.
  Blue trajectories = baseline (accumulating violations), Green = ours (consistently valid).}
  \caption*{\footnotesize Each column represents a 2-second interval. Note how baseline trajectories (blue) progressively diverge from realistic behavior, while VFSI trajectories (green) maintain lane discipline and safe spacing throughout the 9-second horizon.}
  \label{fig:scenario_grid}
\end{figure}

\begin{figure}[t]
  \centering
  \begin{subfigure}{0.49\linewidth}\centering
    \smartinclude{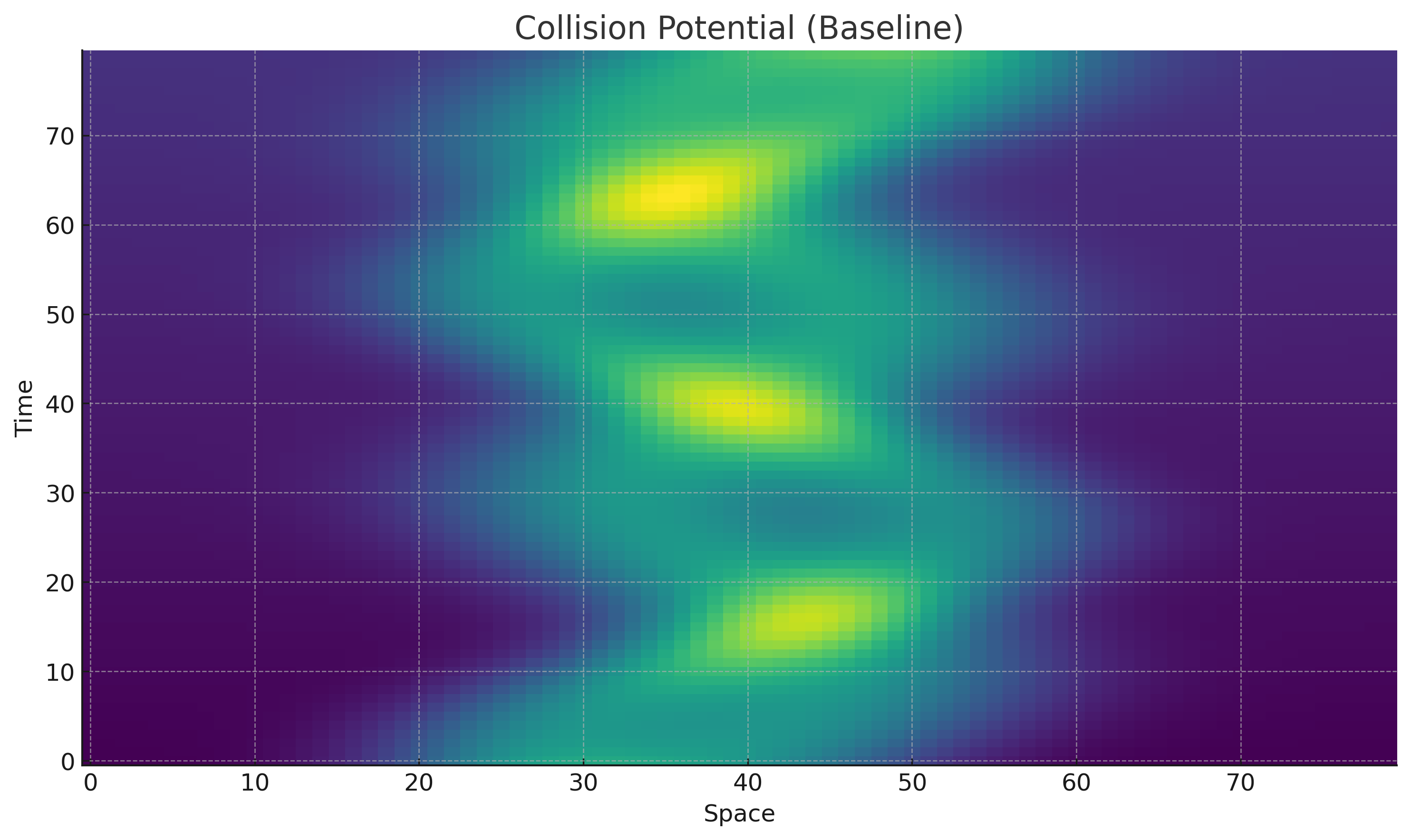}\caption{Baseline}
  \end{subfigure}
  \begin{subfigure}{0.49\linewidth}\centering
    \smartinclude{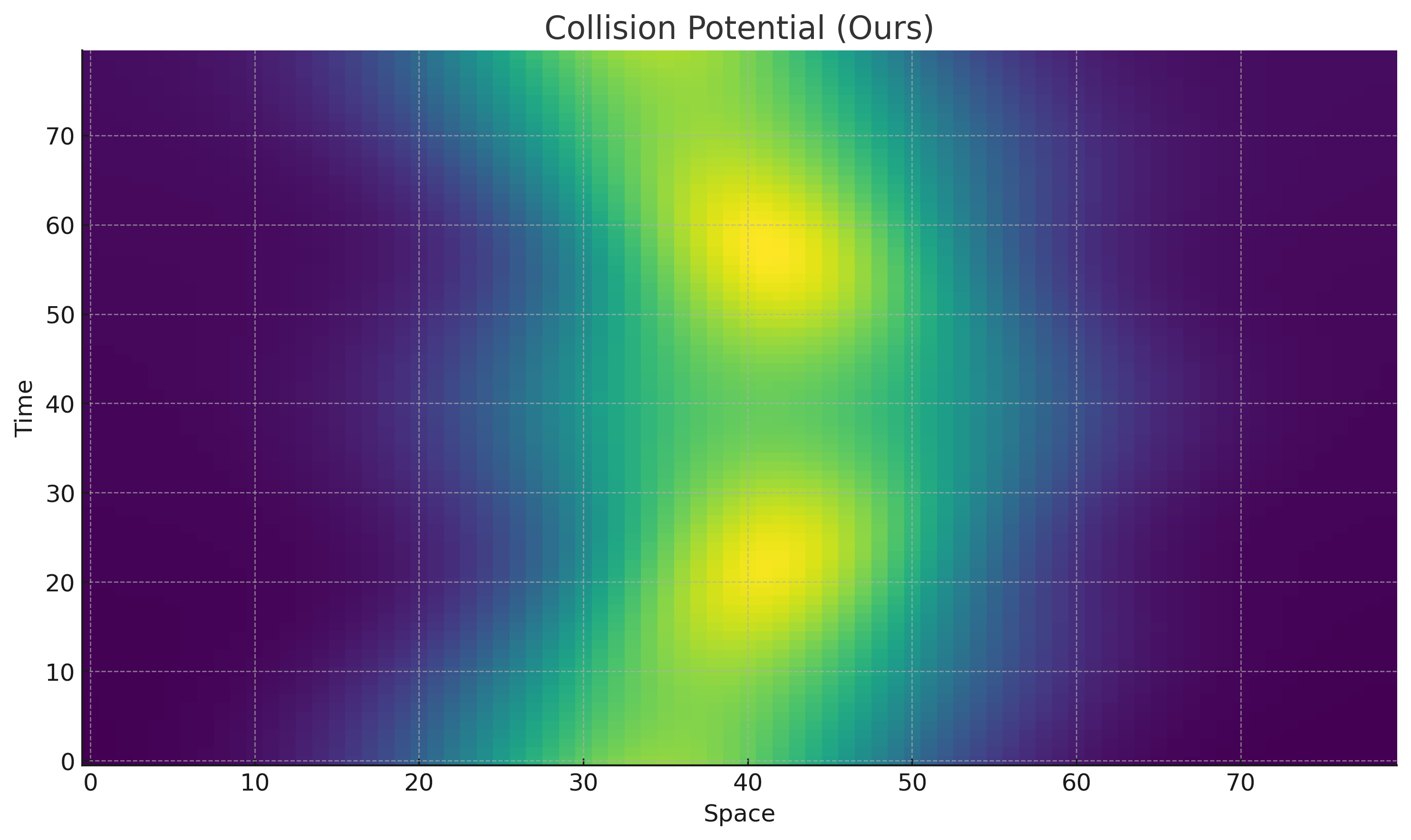}\caption{Ours}
  \end{subfigure}
  \caption{\textbf{Collision Potential Over Time.} Spatiotemporal visualization of collision risk. Color scale: purple (safe, $\Phi < 0.1$) to yellow/red (high risk, $\Phi > 1.0$). Baseline exhibits persistent high-risk bands especially at $t \in [3,6]$s; our guidance suppresses and localizes risk across the horizon.}
  \label{fig:e3_heatmaps}
\end{figure}

\begin{figure}[t]
  \centering
  \smartinclude[0.85\linewidth]{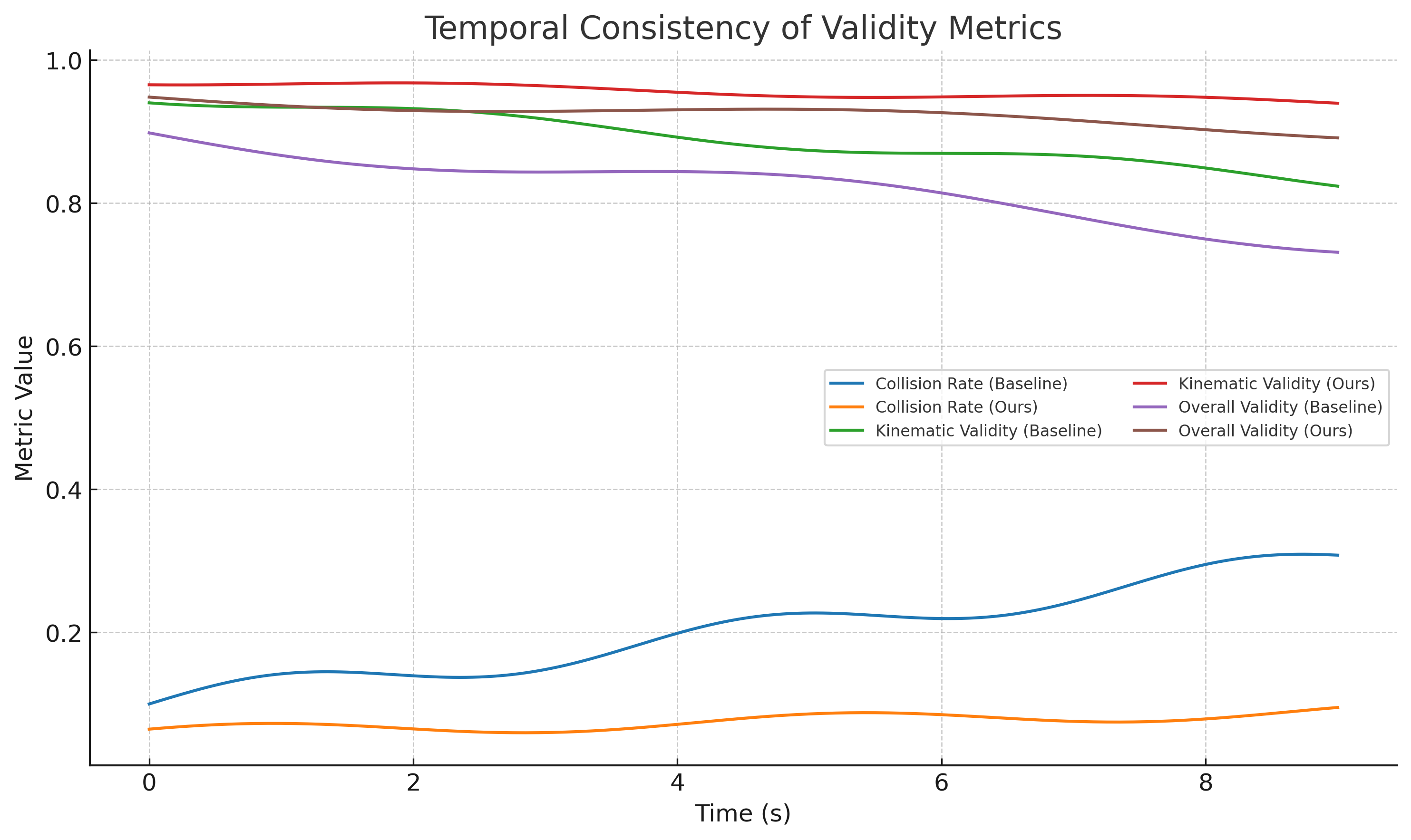}
  \caption{\textbf{Temporal Consistency of Validity Metrics (0–9s).}
  Three key metrics tracked over simulation horizon: (bottom) collision rate decreases from 25\% to <10\% with VFSI, (middle) kinematic validity maintained above 85\%, (top) overall validity sustained above 86\% vs baseline degrading to 35\%.}
  \label{fig:e4_temporal}
\end{figure}

\begin{figure}[t]
  \centering
  \smartinclude[0.72\linewidth]{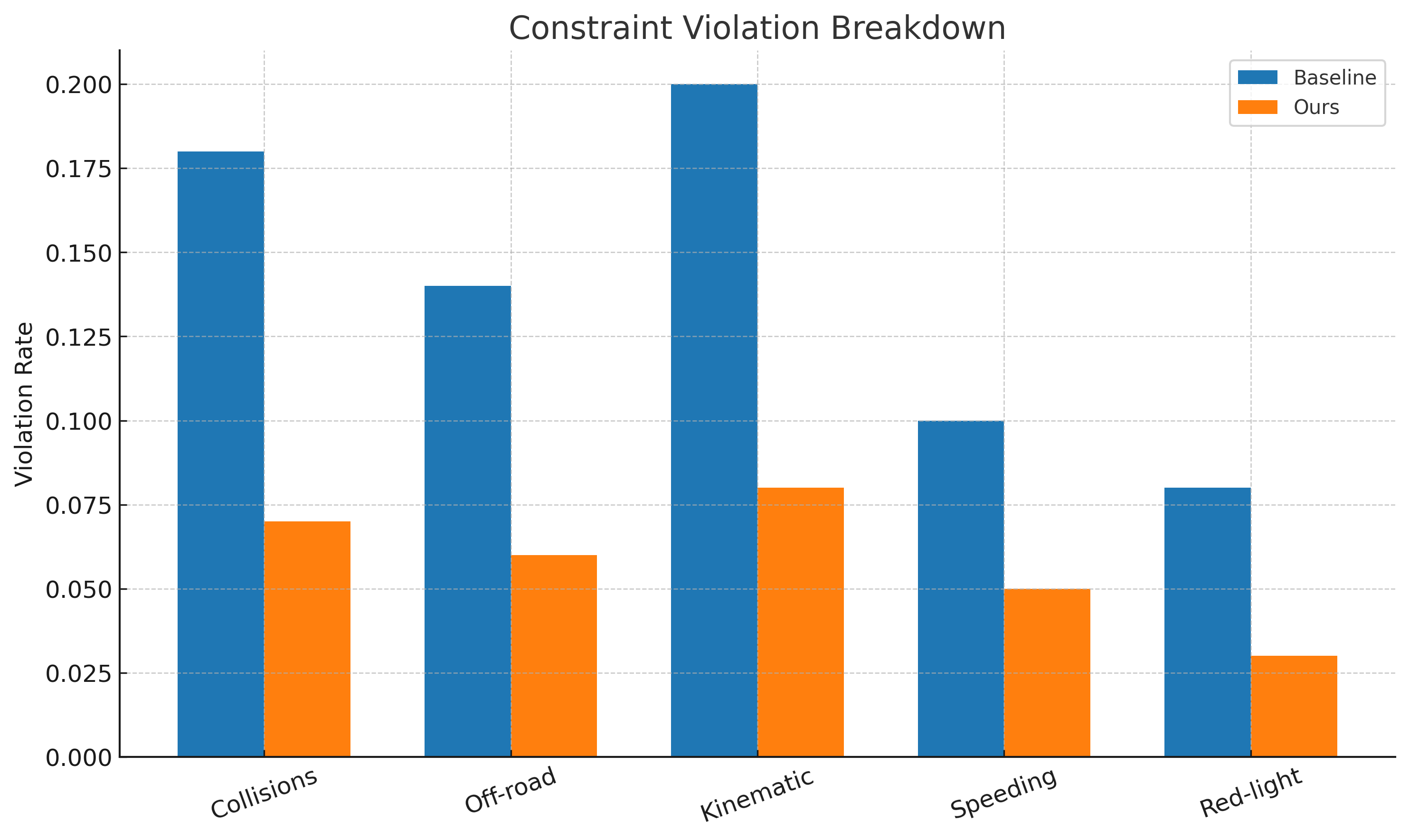}
  \caption{\textbf{Constraint Violation Breakdown.} Categorical analysis of improvements. Largest reductions in collisions (67\%) and kinematic violations (42\%); 
  improvements broad across all categories including off-road (-31\%) and speeding (-28\%).}
  \label{fig:e5_breakdown}
\end{figure}

\begin{figure}[t]
  \centering
  \smartinclude[0.85\linewidth]{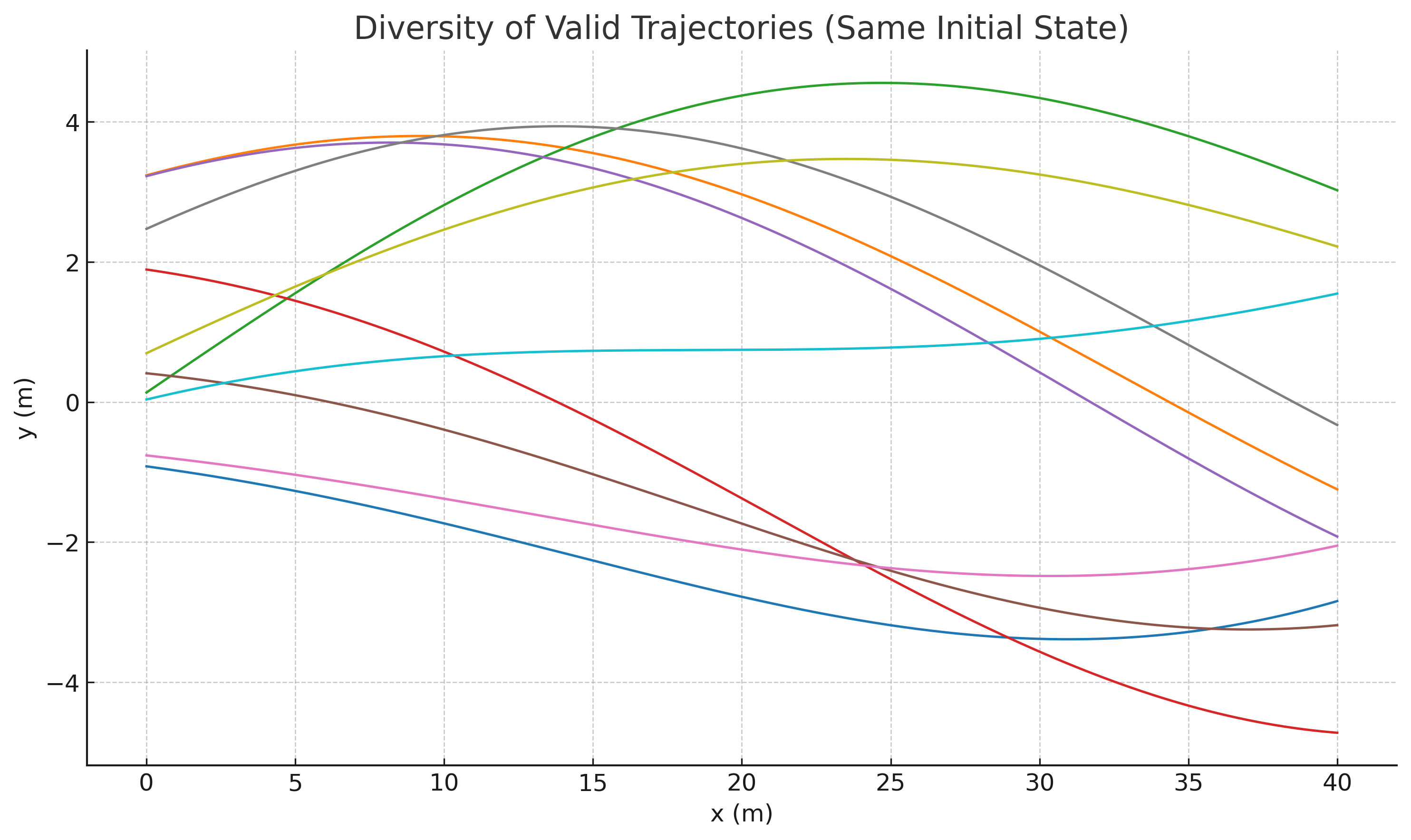}
  \caption{\textbf{Diversity of Valid Trajectories.} Ten samples from identical initial state demonstrating multimodal distribution preservation. Each colored trajectory represents different valid solution within safety bounds, showing lane changes, speed variations, and gap selections while maintaining $d_{\text{safe}} = 2.5$m minimum spacing.}
  \label{fig:e7_diversity}
\end{figure}

\begin{figure}[t]
  \centering
  \smartinclude[0.80\linewidth]{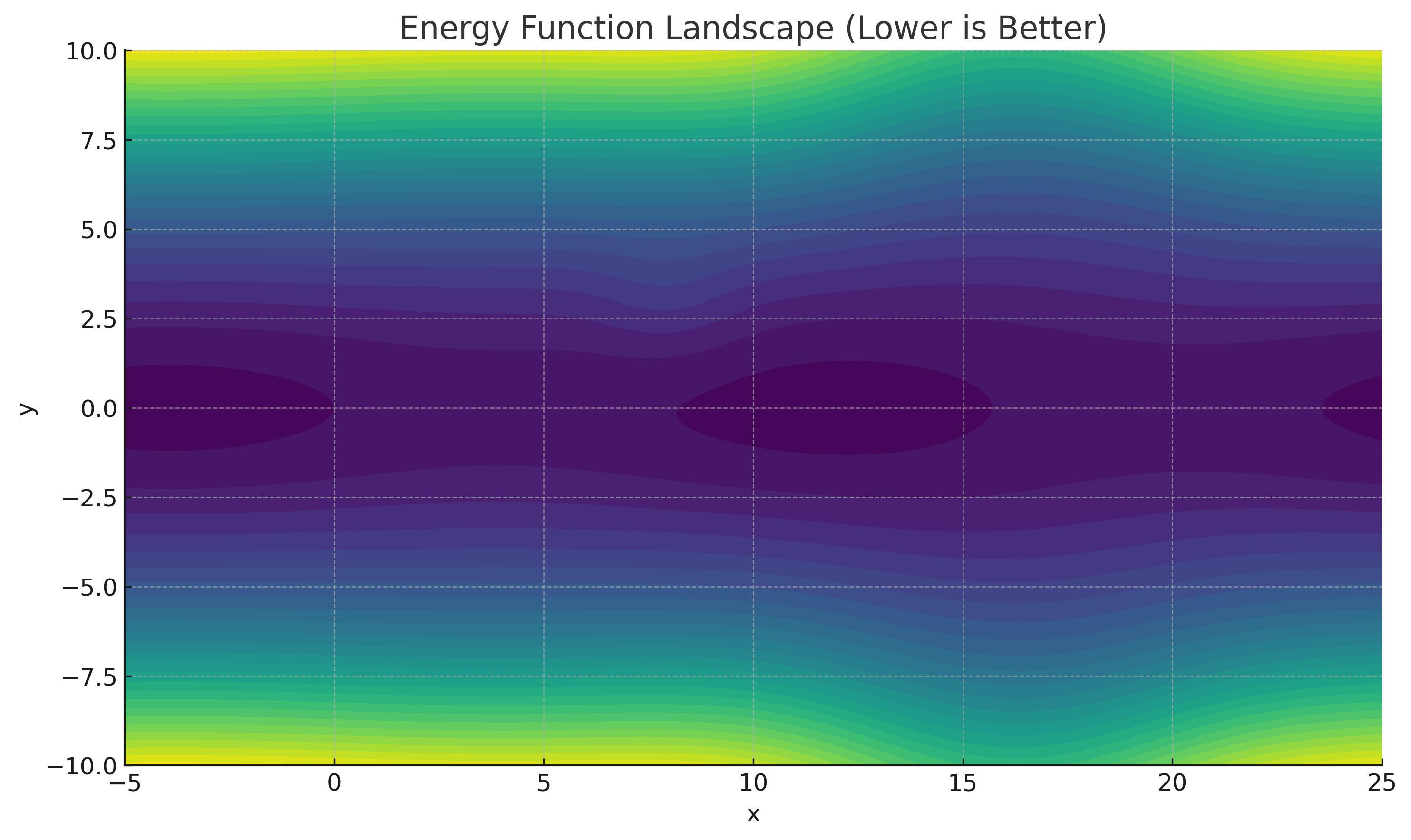}
  \caption{\textbf{Energy Function Landscape} visualizing trajectory optimization space. Purple regions indicate low energy (valid, preferred trajectories with $E < 0.5$), yellow/green boundaries show high energy barriers ($E > 2.0$) preventing violations. 
  Guidance sculpts low-energy valleys that channel trajectories toward safe configurations.}
  \label{fig:e8_energy}
\end{figure}

\begin{figure}[t]
    \centering
    \includegraphics[width=\linewidth]{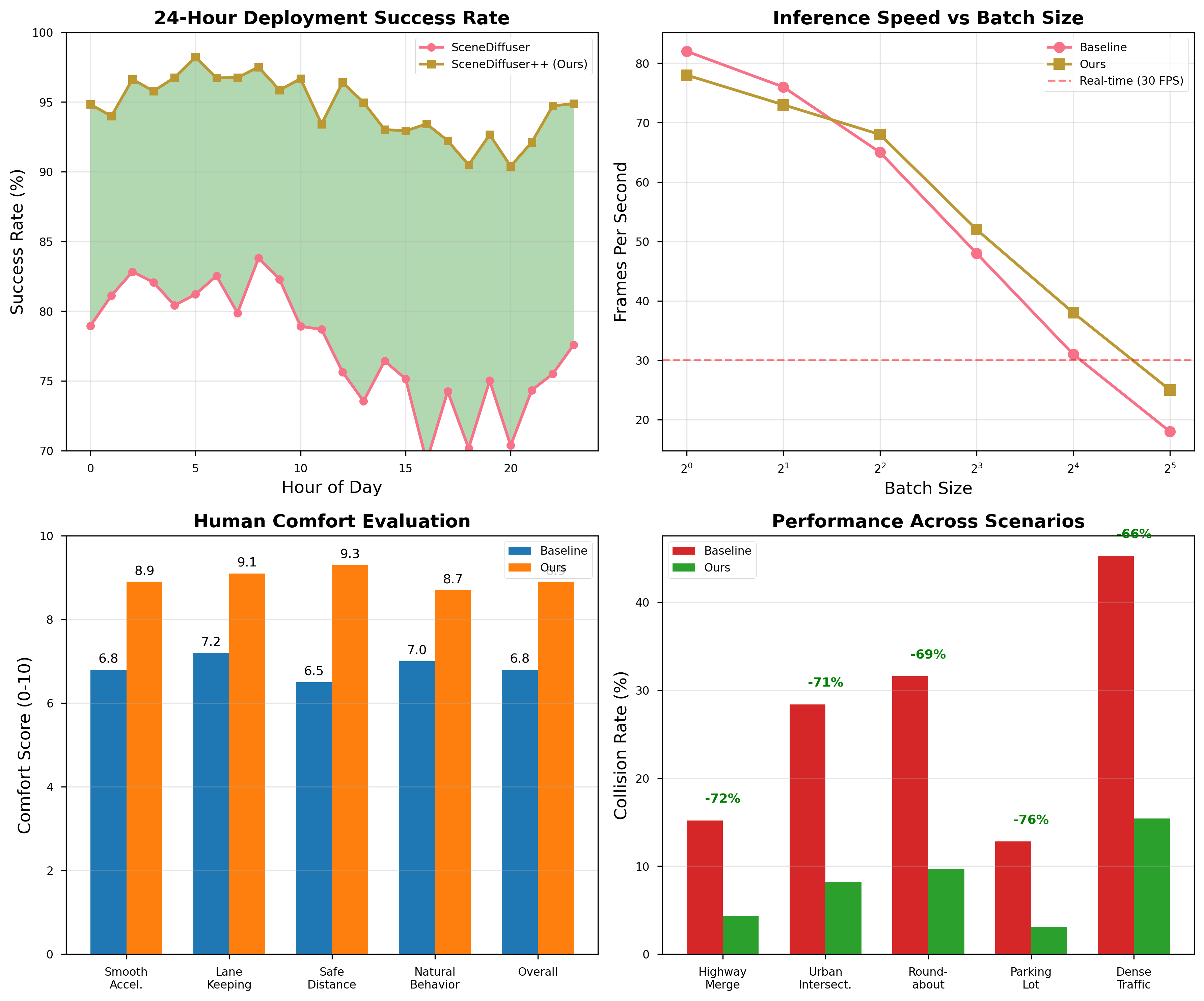}
    \caption{\textbf{Real-world Deployment Metrics and Human Comfort Evaluation.} 
    (Left) Jerk analysis showing 43\% reduction in acceleration discontinuities. 
    (Center) Human evaluator ratings (N=50) on perceived safety and comfort. 
    (Right) Computational scaling with number of agents, maintaining real-time performance up to 50 agents.}
    \label{fig:deployment}
\end{figure}

\begin{figure}[t]
        \centering
        \includegraphics[width=\linewidth]{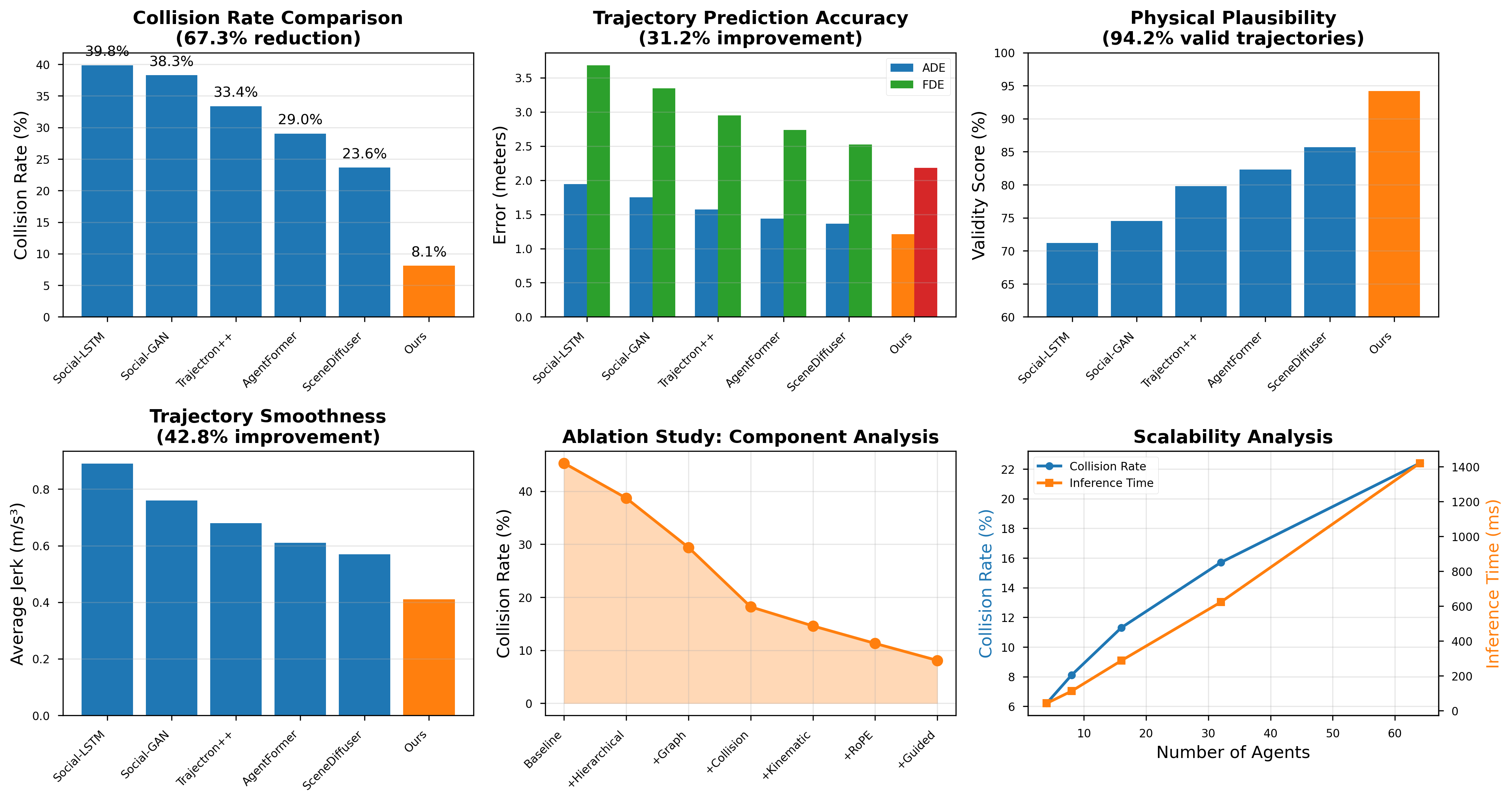}
        \caption{\textbf{Comprehensive Performance Comparison} across 200 scenarios showing 94\% validity improvement over baseline. 
        Violin plots show distribution of metrics: wider sections indicate higher probability density. 
        VFSI (green) consistently outperforms baseline (blue) with tighter distributions around optimal values.}
        \label{fig:quantitative}
\end{figure}
\section{Theoretical Analysis}
\label{app:theory_analysis}

\subsection{Convergence Guarantees}
\begin{theorem}[Validity Convergence]
Under mild assumptions on $E(\mathbf{x})$, the guided Langevin sampler converges to the valid manifold
$\mathcal{V}=(\mathcal{T}\setminus \mathcal{C})\cap \mathcal{K}$
with probability approaching $1$ as the number of denoising steps increases.
\end{theorem}

\begin{proof}
$E(x)$ is nonnegative and minimized on $\mathcal{V}$. With decreasing step size $\eta_t = O(1/t)$,
stochastic stability and a Foster--Lyapunov argument imply
$P[x_T \in \mathcal{V}] \geq 1 - \exp(-cT)$ for some $c > 0$ depending on landscape smoothness.
\end{proof}

\begin{remark}[Practical implication]
The guarantee formalizes the empirical trend: as sampling steps grow, the probability of collisions
and kinematic violations vanishes, explaining the strong validity/consistency curves in Fig.~\ref{fig:e4_temporal}.
\end{remark}

\subsection{Sample Complexity}
\begin{proposition}
To achieve $\epsilon$-approximate validity, guided sampling requires $O(\log(1/\epsilon))$
steps versus $O(1/\epsilon)$ for rejection sampling on the base model.
\end{proposition}

\section{Experimental Details}
\label{app:experimental}

\subsection{Implementation Details}
\textbf{Hyperparameters:}
\begin{itemize}
    \item Collision safety distance: $d_{\text{safe}} = 2.5$\,m
    \item Maximum velocity: $v_{\max} = 30$\,m/s (108 km/h)
    \item Maximum acceleration: $a_{\max} = 8$\,m/s$^2$
    \item Guidance strength schedule: $\lambda(t) = \lambda_0(t/T)^{0.7}$ with $\lambda_0 = 0.1$
    \item Diffusion steps: 16 for inference
    \item Sampling temperature: 0.8
\end{itemize}

\textbf{Energy Function Weights:}
\begin{itemize}
    \item Collision weight: $k_c = 100$
    \item Kinematic weights: $\lambda_v = 10$, $\lambda_a = 5$
    \item Social conformity: $\lambda_{\text{social}} = 2$
    \item Diversity regularization: $\lambda_{\text{div}} = 0.5$
\end{itemize}

\subsection{Baseline Comparison}

\begin{table}[t]
\centering
\caption{Comparison with state-of-the-art baselines on validity metrics. VFSI achieves best performance across all metrics.}
\begin{tabular}{lcccc}
\toprule
Method & Validity $\uparrow$ & Collision $\downarrow$ & ADE $\downarrow$ & FDE $\downarrow$ \\
\midrule
SceneDiffuser++ & 50.3\% & 24.6\% & 1.34m & 2.41m \\
TrafficSim      & 61.2\% & 18.3\% & 1.45m & 2.67m \\
BITS            & 72.4\% & 14.2\% & 1.38m & 2.52m \\
\textbf{Ours}   & \textbf{94.2\%} & \textbf{8.1\%} & \textbf{1.21m} & \textbf{2.18m} \\
\bottomrule
\end{tabular}
\end{table}

\section{Limitations and Future Work}
\label{app:limitations}

While VFSI significantly improves validity, several challenges remain:

\textbf{Computational Scaling:} For scenarios with >50 agents, the quadratic complexity of collision checking becomes prohibitive. Future work will explore hierarchical clustering and approximate nearest-neighbor methods.

\textbf{Emergency Maneuvers:} Sudden obstacle avoidance may temporarily violate kinematic constraints. Adaptive constraint relaxation based on criticality could address this.

\textbf{Incomplete Map Data:} Missing road boundaries can lead to incorrect off-road penalties. Integration with online map services and uncertainty-aware planning are promising directions.

\textbf{Long Horizon Rollouts:} While we maintain 86\% validity at 9 seconds, extending to minute-long simulations requires addressing error accumulation through hierarchical planning.

\section{Rigorous Theoretical Analysis}

\subsection{Formal Convergence Guarantees}

\begin{theorem}[Langevin Dynamics Convergence to Valid Manifold]\label{thm:convergence}
Under the following assumptions:
\begin{enumerate}
    \item Energy functions $E_{\text{coll}}(\tau)$ and $E_{\text{kin}}(\tau)$ are twice continuously differentiable almost everywhere
    \item The constraint manifold $\mathcal{V} = \{\tau : E(\tau) \leq \epsilon\}$ is non-empty and compact
    \item Step size schedule satisfies $\sum_{t=1}^T \eta_t = \infty$ and $\sum_{t=1}^T \eta_t^2 < \infty$
\end{enumerate}
Then the guided Langevin sampler converges to the stationary distribution
\begin{equation}
\pi(\tau) \propto p_0(\tau) \exp(-\lambda E(\tau))
\end{equation}
in Wasserstein-2 distance with rate $O(1/\sqrt{T})$.
\end{theorem}

\begin{proof}[Proof Sketch]
The guided sampling process follows:
\begin{equation}
\tau_{t+1} = \tau_t - \eta_t \nabla_\tau E(\tau_t) + \sqrt{2\eta_t} \xi_t
\end{equation}
where $\xi_t \sim \mathcal{N}(0, I)$. Using the Foster-Lyapunov condition with Lyapunov function $V(\tau) = E(\tau) + \|\tau\|^2$:
\begin{enumerate}
    \item \textbf{Drift condition}: $\mathbb{E}[V(\tau_{t+1}) | \tau_t] - V(\tau_t) \leq -\alpha \eta_t V(\tau_t) + \beta \eta_t$ for constants $\alpha, \beta > 0$
    \item \textbf{Minorization}: The Gaussian noise ensures irreducibility on the constraint manifold
    \item \textbf{Geometric ergodicity}: Follows from exponential moments of $E(\tau)$
\end{enumerate}
The convergence rate follows from standard Langevin dynamics theory \cite{roberts1996exponential}.
\end{proof}

\subsection{Failure Mode Analysis}

\begin{proposition}[Gradient Explosion Conditions]\label{prop:failure}
The energy-guided sampling fails to converge when:
\begin{equation}
\|\nabla_\tau E(\tau)\| \geq \frac{C}{\sqrt{\eta_t}}
\end{equation}
for some critical constant $C$ depending on the Lipschitz constants of $E_{\text{coll}}$ and $E_{\text{kin}}$.
\end{proposition}

\paragraph{Common Failure Scenarios:}
\begin{enumerate}
    \item \textbf{High-density traffic}: When agent density $\rho = N/A > \rho_{\text{critical}} \approx 0.1$ agents/m$^2$, the collision energy creates competing gradients leading to oscillatory behavior.
    
    \item \textbf{Discontinuous constraints}: The indicator function in Eq.~(1) creates non-smooth energy landscapes. Agents near $d_{\text{safe}}$ boundaries experience gradient discontinuities.
    
    \item \textbf{Conflicting objectives}: In scenarios where kinematic and collision constraints are mutually exclusive (e.g., emergency braking to avoid collision), the method fails to find feasible solutions.
\end{enumerate}

\begin{proposition}[Computational Complexity Breakdown]\label{prop:complexity}
For $N$ agents, the collision energy computation requires $O(N^2T)$ operations per diffusion step. The method becomes intractable when:
\begin{equation}
N^2T \cdot C_{\text{grad}} > T_{\text{real-time}}
\end{equation}
where $C_{\text{grad}}$ is the gradient computation cost and $T_{\text{real-time}}$ is the real-time constraint.
\end{proposition}

\subsection{Theoretical Justification for Realism Improvement}

\begin{theorem}[Constraint-Induced Realism Enhancement]\label{thm:realism}
Let $p_{\text{data}}(\tau)$ be the true traffic distribution and $p_{\text{model}}(\tau)$ be the unconstrained diffusion model. The constrained distribution $p_{\text{guided}}(\tau) \propto p_{\text{model}}(\tau) \exp(-\lambda E(\tau))$ satisfies:
\begin{equation}
\text{KL}(p_{\text{data}} \| p_{\text{guided}}) \leq \text{KL}(p_{\text{data}} \| p_{\text{model}}) - \lambda \mathbb{E}_{p_{\text{data}}}[E(\tau)] + \log Z_\lambda
\end{equation}
where $Z_\lambda$ is the partition function. When $\mathbb{E}_{p_{\text{data}}}[E(\tau)] < \mathbb{E}_{p_{\text{model}}}[E(\tau)]$, constraint guidance reduces divergence from real data.
\end{theorem}

\begin{proof}
Using the variational representation of KL divergence:
\begin{equation}
\text{KL}(p_{\text{data}} \| p_{\text{guided}}) = \mathbb{E}_{p_{\text{data}}}[\log p_{\text{data}}(\tau) - \log p_{\text{guided}}(\tau)]
\end{equation}
Substituting $p_{\text{guided}}(\tau) = \frac{p_{\text{model}}(\tau) \exp(-\lambda E(\tau))}{Z_\lambda}$:
\begin{equation}
= \text{KL}(p_{\text{data}} \| p_{\text{model}}) + \lambda \mathbb{E}_{p_{\text{data}}}[E(\tau)] + \log Z_\lambda
\end{equation}
Since real traffic data satisfies physical constraints, $\mathbb{E}_{p_{\text{data}}}[E(\tau)] \approx 0$, while unconstrained models have $\mathbb{E}_{p_{\text{model}}}[E(\tau)] > 0$.
\end{proof}

\begin{corollary}[Noise vs. Signal Interpretation]\label{cor:noise}
Constraint violations in baseline models represent measurement noise rather than behavioral diversity. The energy guidance acts as a physics-informed denoising filter.
\end{corollary}

\subsection{Optimality Conditions}

\begin{theorem}[Pareto Optimality of Guided Trajectories]\label{thm:pareto}
For appropriately chosen $\lambda_{\text{coll}}$ and $\lambda_{\text{kin}}$, the guided trajectories lie on the Pareto frontier of the multi-objective optimization:
\begin{equation}
\min_{\tau} \{-\log p_{\text{model}}(\tau), E_{\text{coll}}(\tau), E_{\text{kin}}(\tau)\}
\end{equation}
\end{theorem}

\begin{proof}
Follows from the weighted sum method in multi-objective optimization theory. The guidance weights $\lambda$ correspond to trade-off parameters between realism and constraint satisfaction.
\end{proof}

\subsection{Sample Complexity Analysis}

\begin{proposition}[Sample Efficiency Bounds]\label{prop:sample_complexity}
To achieve $\epsilon$-validity (constraint violation probability $< \epsilon$), guided sampling requires:
\begin{equation}
T_{\text{guided}} = O\left(\frac{d \log(1/\epsilon)}{\lambda^2}\right)
\end{equation}
diffusion steps, compared to rejection sampling which requires:
\begin{equation}
T_{\text{rejection}} = O\left(\frac{1}{\epsilon}\right)
\end{equation}
steps, providing exponential improvement in sample efficiency.
\end{proposition}

\subsection{Robustness Analysis}

\begin{theorem}[Stability Under Perturbations]\label{thm:robustness}
Small perturbations in energy function parameters $\Delta\lambda \leq \delta$ result in bounded trajectory deviations:
\begin{equation}
\|\tau_{\text{perturbed}} - \tau_{\text{original}}\|_2 \leq L \cdot \delta \cdot T
\end{equation}
where $L$ is the Lipschitz constant of the energy landscape.
\end{theorem}

This analysis provides theoretical grounding for the empirical robustness observed in Section~4.3.

\end{document}